\documentclass[onefignum,onetabnum]{siamonline220329}
\usepackage{pdfpages}

\usepackage{lipsum}
\usepackage{amsfonts, amssymb, bbm, amsmath}
\usepackage{setspace}
\usepackage{aligned-overset}
\usepackage{graphicx}
\usepackage{epstopdf}
\usepackage{algorithm}
\usepackage{algorithmicx}
\usepackage{algpseudocode}
\usepackage{multirow}

\ifpdf
  \DeclareGraphicsExtensions{.eps,.pdf,.png,.jpg}
\else
  \DeclareGraphicsExtensions{.eps}
\fi

\usepackage{enumitem}
\setlist[enumerate]{leftmargin=.5in}
\setlist[itemize]{leftmargin=.5in}

\def\ddefloop#1{\ifx\ddefloop#1\else\ddef{#1}\expandafter\ddefloop\fi}
\def\ddef#1{\expandafter\def\csname v#1\endcsname{\ensuremath{\boldsymbol{#1}}}}
\ddefloop abcdefghijklmnopqrstuvwxyzABCDEFGHIJKLMNOPQRSTUVWXYZ\ddefloop
\def\ddef#1{\expandafter\def\csname v#1\endcsname{\ensuremath{\boldsymbol{\csname #1\endcsname}}}}
\ddefloop {alpha}{beta}{gamma}{delta}{epsilon}{varepsilon}{zeta}{eta}{theta}{vartheta}{iota}{kappa}{lambda}{mu}{nu}{xi}{pi}{varpi}{rho}{varrho}{sigma}{varsigma}{tau}{upsilon}{phi}{varphi}{chi}{psi}{omega}{Gamma}{Delta}{Theta}{Lambda}{Xi}{Pi}{Sigma}{varSigma}{Upsilon}{Phi}{Psi}{Omega}{ell}\ddefloop

\DeclareMathOperator*{\E}{\mathbb{E}}
\DeclareMathOperator*{\p}{\mathbb{P}}

\DeclareMathOperator*{\argmin}{arg\,min}
\DeclareMathOperator{\col}{col}
\DeclareMathOperator{\supp}{supp}

\DeclareMathOperator{\op}{op}

\DeclareMathOperator\LLE{LLE}

\newcommand{\minspace}{\mathcal{U}}
\newcommand\T{{\scriptscriptstyle{\mathsf{T}}}}
\newcommand{\vzero}{\boldsymbol{0}}
\newcommand{\id}{\mathbbm{1}}
\newcommand\R{\mathbb{R}}
\newcommand\N{\mathbb{N}}

\newcommand\Normal{\mathcal{N}}
\newcommand\Data{\mathcal{D}}

\newcommand\cO{\mathcal{O}}

\newcommand\ratiomoment{\mu_{\rho}}
\newcommand\linkmoment{\mu_{f}}
\newcommand{\Egrad}[1]{\overline{\nabla}_{#1} f}
\newcommand{\shift}{\vtheta}
\newcommand\deltamoment{\mu_{\nabla}}
\newcommand\hermite{w}

\newcommand{\ESGOP}{\overline{\vM}}
\newcommand{\ASGOP}{\widetilde{\vM}}
\newcommand{\estimator}{\widehat\vM}

\newlength\oversetwidth
\newlength\underwidth
\newcommand\alignedoverset[2]{
  \settowidth\oversetwidth{$\overset{#1}{#2}$}
  \settowidth\underwidth{$#2$}
  \setlength\oversetwidth{\oversetwidth-\underwidth}
  \hspace{.5\oversetwidth}
  &
  \settowidth\oversetwidth{$\overset{#1}{#2}$}
  \settowidth\underwidth{$#2$}
  \setlength\oversetwidth{\oversetwidth-\underwidth}
  \hspace{-.5\oversetwidth}
  \overset{#1}{#2}
}

\newsiamremark{remark}{Remark}
\newsiamremark{hypothesis}{Hypothesis}
\crefname{hypothesis}{Hypothesis}{Hypotheses}
\newsiamthm{claim}{Claim}
\newsiamthm{fact}{Fact}
\newtheorem{assumption}{Assumption}

\headers{Smoothed Gradient Outer Product}{G. Yuan, M. Xu, S. Kpotufe and D. Hsu}

\title{Efficient Estimation of the Central Mean Subspace via Smoothed Gradient Outer Products}

\author{Gan Yuan\thanks{Department of Statistics, Columbia University, New York, NY, 10027
  (\email{gan.yuan@columbia.edu}, \email{skk2175@columbia.edu}).}
\and Mingyue Xu\thanks{Data Science Institute, Columbia University, New York, NY, 10027 (\email{mx2258@columbia.edu}).}
\and Samory Kpotufe\footnotemark[2]
\and Daniel Hsu\thanks{Department of Computer Science and Data Science Institute, Columbia University, New York, NY, 10027 (\email{djhsu@cs.columbia.edu}).}}

\usepackage{amsopn}
\DeclareMathOperator{\diag}{diag}

\makeatletter
\newcommand*{\addFileDependency}[1]{
  \typeout{(#1)}
  \@addtofilelist{#1}
  \IfFileExists{#1}{}{\typeout{No file #1.}}
}
\makeatother

\newcommand*{\myexternaldocument}[1]{%
    \externaldocument{#1}%
    \addFileDependency{#1.tex}%
    \addFileDependency{#1.aux}%
}

\ifpdf
\hypersetup{
  pdftitle={Efficient Estimation of the Central Mean Subspace},
  pdfauthor={G. Yuan, M. Xu, S. Kpotufe and D. Hsu}
}
\fi

\myexternaldocument{ESGOP_supplementary}

\begin{document}

\maketitle

\begin{abstract}
 We consider the problem of \emph{sufficient dimension reduction} (SDR) for multi-index models. The estimators of the central mean subspace in prior works either have slow (non-parametric) convergence rates, or rely on stringent distributional conditions (e.g., elliptical symmetric covariate distribution $P_{\vX}$). In this paper, we show that a fast parametric convergence rate of form $C_d \cdot  n^{-1/2}$ is achievable via estimating the \emph{expected smoothed gradient outer product}, for a general class of distribution $P_{\vX}$ that admits Gaussian or heavier distributions. When the link function is a polynomial with a degree of at most $r$ and $P_{\vX}$ is the standard Gaussian, we show that the prefactor depends on the ambient dimension $d$ as $C_d \propto d^r$.
\end{abstract}



\section{Introduction} \label{sec:intro}

We consider a regression model for data $(\vX,Y) \in \R^d \times \R$ in which the regression function $g(\vx) = \E[Y \mid \vX = \vx]$ has the form
\begin{equation*}
  g(\vx) = f(\vU^\T \vx)
\end{equation*}
for some unknown matrix $\vU \in \R^{d \times k}$ of rank $k$ and some unknown function $f \colon \R^k \to \R$.
The goal of this work is to provide computationally tractable methods for estimating the column space of $\vU$, denoted $\col(\vU)$, from a finite sample of $(\vX,Y)$.

The model described above is called a \emph{multi-index model} in the literature, and $\col(\vU)$ is referred to as the \emph{index space}.
The index space is also known as the \emph{central mean subspace} in the context of sufficient dimension reduction~\cite{cook1994interpretation, cook2009regression} under mild identifiability conditions on $\supp(\vX)$\footnote{See \Cref{sec:indexspace} for a discussion of these conditions.}.
An accurate estimate of the index space has many potential applications, including the following:

\begin{itemize}
  \item
    The index space has, in the context of regression, an explanatory role analogous to that of the leading principal components of the data.
    In particular, it reveals which (linear combinations of) covariates are able to best predict the response $Y$.

  \item If $k \ll d$, then the achievable rates for estimating the regression function $g$ may be greatly improved by first projecting the covariates $\vx$ to the index space before performing estimation.
    For instance, if $g$ is only known to be Lipschitz, then we may hope to improve the rate of estimation from $n^{-1/(d+2)}$ to $n^{-1/(k+2)}$.

\end{itemize}

The problem of estimating the index space has received much attention, but important gaps remain in terms of what approximation error rates are achievable by computationally tractable estimators.
Results may be cataloged into two main directions, one emphasizing fast convergence rates\footnote{Under suitable subspace approximation measures, e.g., angles between spaces (see \Cref{sec:indexspace}).} but requiring strict distributional conditions, the other preferring general distributional settings but admitting significantly worse convergence rates.
For example, on one hand, so-called \emph{inverse regression} approaches---which estimate the index space via functionals of $\E[\vX \mid Y]$---often admit fast convergence rates of the form $O(n^{-1/2})$ however at the cost of restrictive conditions tying the marginal distribution $P_{\vX}$ to $\col(\vU)$; furthermore, outside of these conditions on $P_X$, only a strict subspace of $\col(\vU)$ may be recovered. On the other hand, \emph{forward regression} methods---which leverage functionals of the regression function $g := \E[Y \mid \vX]$,
e.g., its derivatives---can ensure full recovery of the index space, importantly without un-naturally forcing a relation between $P_{\vX}$ and $\col(\vU)$; however such milder distributional conditions often come at the cost of significantly slower estimation rates, often of a nonparametric form $n^{-1/\Omega(k)}\gg n^{-1/2}$. 

Our main contribution is to elucidate the following middle-ground: we show that fast rates of the form $C_d\cdot n^{-1/2}$ remain possible for a reweighted variant of forward methods, for a prefactor $C_d$  depending on data distribution.
While our distributional conditions---knowledge of the marginal distribution $P_{\vX}$ and certain moment assumptions---are stricter than for some forward regression methods, our rate guarantees are of a parametric form even for nonparametric link $g$, similar to usual convergence rates for inverse regression methods. Yet, unlike inverse regression methods, we guarantee full recovery of the index space $\col(\vU)$ under qualitatively milder distributional conditions, where in particular, we impose no condition on the relation between $P_{\vX}$ and the unknown index space $\col(\vU)$ as is common in that literature. 

Our result relies on estimating outer products of a \emph{smoothed} version of $\nabla f$, where the smoothing itself---induced by careful reweighting of the training sample---leverages a first-order Stein's Lemma to guarantee full recovery of $\col(\vU)$.
Similar to other existing analyses, we use the knowledge of $P_{\vX}$ to properly reweight the training sample. 

To better situate our results, we expand next on the state-of-the-art on the subject, followed by a more detailed overview of the results.

\subsection{Prior Theoretical Works}
\label{sec:other_works}

The problem of index space estimation has received much attention over the last few decades from both statistics and machine learning communities~\cite{li91sir, Li92phd, cook00save, Xia02mave, Spokoiny01structure, hristache2001direct, Dalalyan08, babichev18slice, klock21estimating, dudeja2018learning, chen2020learning, mousavi-hosseini2023neural}.
For ease of comparison, we first note that a prevalent approach in the literature, also adopted here, is to proceed by defining and estimating a population object $\vM \in \R^{d \times d}$, whose column space estimates (part of) 
$\col(\vU)$; the object $\vM$ is typically a functional of either the inverse or forward conditional distribution $\vX \mid Y$ or $Y \mid \vX$. 

\paragraph{Inverse regression methods}

As the name suggests, the class of inverse regression methods leverage functionals of the inverse relation $\vX\mid Y$ that may reveal the index space or a subspace thereof.
Although these methods are initially designed to estimate the \emph{central subspace}, they can also serve the purpose of estimating the central mean subspace when the covariates $\vX$ and the error $\varepsilon = Y - g(\vX)$ are independent, under which case these two subspaces are identical.
For instance, Sliced Inverse Regression (SIR) by \cite{li91sir} estimates  $\vM = \text{Cov}[\E[\vX\mid Y]]$ and Sliced Average Variance Estimator (SAVE) by \cite{cook00save} estimates $\vM =\E[(\vI_d - \text{Cov}[\vX\mid Y])^2]$. 
 {By carefully estimating means or covariances over \emph{slices}, i.e., partitions of the $Y$ space, parametric convergence rates of the form $O(n^{-1/2})$ may be guaranteed, albeit under restrictive distributional conditions.} 
 
For instance, SIR imposes a so-called \textit{linear conditional mean} (LCM) condition (where $\E[\vX\mid \vU^{\T}\vX]$ is linear in  $\vU^{\T} \vX$) in order to ensure $\col(\vM) \subset \col(\vU)$, i.e., in order to guarantee \emph{partial} recovery of $\col(\vU)$; if in addition we impose a so-called \textit{random conditional projection} (RCP) condition (requiring for all $v\in \R^d$ that $\E[\vv^{\T}\vX\mid Y]$ is not constant as a r.v.) then one can ensure $\col(\vM) = \col(\vU)$. These distributional conditions, by restricting the relation between $P_{\vX}$ and the index space $\col(\vU)$, rule out some rather benign settings: for instance, for $k = 1$, it is well known that under Gaussian $P_{\vX}$, these conditions cannot hold for \emph{even} link functions such as $f(z) = z^2$ (also referred as phase retrieval problem). Similar examples of failure modes are discussed for SAVE in \Cref{app:save}. Additional discussions and examples for non-exhaustiveness issues can be found in~\cite{dudeja2018learning} and~\cite{klock21estimating}.

\paragraph{Forward regression methods}

In contrast to inverse regression methods, the class of forward regression methods leverage functionals of the regression function $\E[Y \mid \vX]$.
An important line of work relies on \textit{Stein's Lemma} in constructing $\vM$. 
The earliest work along this line by \cite{Brillinger82glm} considered a single-index model ($k=1$)
estimates the index space direction, say $\vu \in \R^d$, via $\vM = \E[Y \cdot \vX]$: suppose that $P_{\vX}$ were standard Gaussian, and assuming an additive noise model, then Stein's Lemma implies that $\E[Y \cdot \vX] =  \E[f'(\vu^{\T} \vX)] \cdot \vu$. 
For the general case $k > 1$, \cite{Yang17stein} proposed an estimator based on a second-order Stein's Lemma that is $\sqrt{n}$-consistent. Similar to our work, they assume the setting on known $P_X$, and admit a broad class of distributions including Gaussian and more heavy-tailed distributions. 
However, their theoretical results relies on the non-singularity of the expected Hessian $\E[\nabla^2 g(\vX)]$, again ruling out some apparently benign settings: for instance for $k =1$, and $\E X = 0$, simple link functions such as $f(z) = z^3$ are not admitted. In contrast, the approach analyzed here relies on a first-order Stein's Lemma (\Cref{lem:stein}) and thus imposes no condition on the Hessian or any other higher order derivatives of $g$.  

Another class of forward regression methods are those that rely on functionals of the gradient of $g$.
In the simplest case of single index models ($k=1$), we may take $\vM = \E[\nabla g(\vX)] = \E[f'(\vu^{\T} \vX)] \cdot \vu$, for an index direction $\vu$; thus, provided $\E[f'(\vu^{\T} \vX)]$ does not vanish (similar to \cite{Brillinger82glm} discussed above), \cite{hristache2001direct} designs an estimator of $\vu$ with convergence rate $n^{-1/2}$.
In the more general case $k \geq 1$, one may consider a so-called 
\emph{expected gradient outer product} (EGOP), defined as $\vM = \E[\nabla g(\vX)\nabla g(\vX)^{\T}] = \vU \E[\nabla f(\vU^{\T}\vX) \nabla f(\vU^{\T}\vX)^{\T}] \vU^{\T}$.
This approach is quite appealing as it holds that, under quite mild conditions (see \Cref{sec:indexspace}), one can show that $\col(\vM) = \col(\vU)$. Thus, a number of estimators of the EGOP $\vM$ have been proposed over the years, e.g., \cite{Samarov93functional,Xia02mave,Trivedi14EGOP,radhakrishnan2022feature}.
However, the convergence guarantees available for these methods are, at best, of the nonparametric form $n^{-1/\Omega(k)}$, without additional assumptions on $P_{\vX}$.
For the restricted case of $k \leq 3$, \cite{Spokoiny01structure, Dalalyan08} shows that the parametric rate $n^{-1/2}$ may be achieved under known Gaussian noise and in a fixed design setting (under some conditions on the observed $\vX$ data itself).
We have no such restriction on the index space dimension $k$, and operate in the more general random design setting, however while assuming known $P_{\vX}$. 

Finally, the popular MAVE approach of \cite{Xia02mave}, while still a forward regression method, differs considerably from all the above in that it does not explicitly define a surrogate object $\vM$ for $\vU$; instead, it works by optimizing a suitable objective over all possible projections $\vU$. MAVE is appealing in that, similar to EGOP, it can guarantee full recovery of the index space under mild conditions.
However, this also comes at the cost of a nonparametric rate of the form $n^{-1/\Omega(k)}$.

\begin{table}[tbhp]
    \setstretch{1.2}
    \footnotesize
    \centering
    \caption{Various Methods of Estimating CMS with Their Key Assumptions and Theoretical Guarantee.}
    \label{tab:literature}
    \begin{tabular}{ccccc }
      \multirow{2}{*}{\bf Methods} & \bf Single/Multiple & {\bf Key} & {\bf Error} &  {\bf Prefactor with} \\
      & \bf Index & \bf Assumptions & \bf Rate & \bf Polynomial Link\\\hline
       SIR\cite{li91sir}  & Multiple  & LCM, RCP &  $O(n^{-1/2})$ & NA \\\hline
       SAVE\cite{cook00save}  & Multiple & LCM, RCP &  $O(n^{-1/2})$ & NA \\   \hline
       MAVE\cite{Xia02mave}& Multiple & Smooth Link Functions & $O(n^{-1/\Omega(k)})$ & NA \\\hline
       \multirow{2}{*}{SAA\cite{Spokoiny01structure,hristache2001direct}}  &  \multirow{2}{*}{Multiple} &  \multirow{2}{*}{Smooth Link Functions} & $O(n^{-1/2})$ &\multirow{2}{*}{NA}  \\
        & & & for $k \le 3$ \\\hline
       SOSM\cite{Yang17stein}   &  Multiple & Non-degenerate Hessian &  $O(n^{-1/2})$ & NA \\\hline
       ADE\cite{Brillinger82glm} & Single & Gaussian Covariates &   $O(n^{-1/2})$ & NA \\\hline
       \multirow{2}{*}{EIVH\cite{dudeja2018learning}}&\multirow{2}{*} {Single} & Gaussian Covariates& \multirow{2}{*}{$O(n^{-1/2})$} & \multirow{2}{*}{NA}\\
       & & Smooth Link Function & \\\hline
       \multirow{2}{*}{ShallowNet\cite{bietti22learning}}&\multirow{2}{*} {Single} & Gaussian Covariates& \multirow{2}{*}{$O(n^{-1/2})$} & \multirow{2}{*}{$C_d \sim d^{r/2}$}\\
       & & Smooth Link Function & \\\hline
        \multirow{2}{*}{SmoothedSGD\cite{damian2023smoothing}}&\multirow{2}{*} {Single} & Gaussian Covariates & \multirow{2}{*}{$O(n^{-1/2})$} & \multirow{2}{*}{$C_d \sim d^{r/4}$}\\
       & & Known Link Function & \\\hline
       {ESGOP}&\multirow{2}{*} {Multiple} & Known Density Ratio & \multirow{2}{*}{$O(n^{-1/2})$} & \multirow{2}{*}{$C_d \sim d^{r}$}\\
        (Our Method)& & Smooth Link Function & \\\hline
    \end{tabular}
\end{table}

 {
 Finally, for easier reference, we list the key assumptions and theoretical guarantee of some popular works in Table \ref{tab:literature}.}
\subsection{Detailed Contributions}
We introduce an object called the \emph{Expected Smoothed Gradient Outer Product (ESGOP)} $\ESGOP$, which can be viewed as a surrogate for the EGOP.
The ESGOP is based on a \emph{smoothed} version of the gradient $\nabla g(\vX)$, as well as a subtle change of distribution (from $P_{\vX}$ to a Gaussian) that enables application of Stein's inequality. 

We first show that $\ESGOP$ is \emph{exhaustive} in that $\col(\ESGOP) = \col(\vU)$ (\Cref{prop:exhaust}). The surrogate $\ESGOP$ can then be estimated from a finite sample from $P_{\vX, Y}$ via a reweighting scheme; in particular, the estimation of the \emph{smoothed} gradients may be interpreted as local linear estimator with a data-dependent kernel (\Cref{rmk:llr}).

Our main result states that, for a general class of distributions $P_{\vX}$ admitting Gaussians and heavier tail distributions, our estimator admits a parametric error rate of the form $C_d \cdot n^{-1/2}$.
Finally, to elucidate the leading constant $C_d$, we consider the special case of polynomial links $g$---which has received much recent interest---and show that it behaves as $d^{O(r)}$ for polynomials of degree at most $r$ (\Cref{cor:poly-rate}), in line with recent lower-bounds of \cite{damian2022neural} for certain classes of computationally tractable procedures. 

\subsection{Outline of the Paper}

The remainder of the paper is organized as follows. In \Cref{sec:preliminaries}, we introduce the basic setups and some notations. In \Cref{sec:results}, we propose a novel estimator for central mean subspace based on the expected smoothed gradient outer product, and provide upper-bounds for the estimation errors for general link functions and polynomial link functions. In \Cref{sec:analysis}, we outline the analysis of the main results presented in \Cref{sec:results}. In \Cref{sec:experiment}, we demonstrate the choice of hyperparameters in the main algorithm through a simulation study. Omitted proofs can be found in the Supplementary Material.

\section{Preliminaries} \label{sec:preliminaries}

\subsection{Notations}

Throughout, we use bold font letters (e.g. $\vx$, $\vX$) to refer to vectors and matrices. 
Let $\langle\cdot,\cdot\rangle$ and $\lVert \cdot \rVert$ denote Euclidean inner product and  norm on $\R^d$.
For a probability measure $\nu$ on $\R^d$, we let $\langle f,g\rangle_{\nu} := \E_{\vX\sim\nu}[f(\vX)\cdot g(\vX)]$ and $\lVert f\rVert_{\nu} := \sqrt{\langle f,f\rangle_{\nu}}$ denote the inner-product and norm on the function space $L_{2}(\nu)$. 
The operator norm of a matrix $\vM$ (with respect to Euclidean norm) is denoted by $\lVert\vM\rVert_{\op}$.
For a positive integer $d$, we use $[d]$ to denote the set $\{1,\ldots,d\}$.

\subsection{Central Mean Subspace} \label{sec:indexspace}

Let $(\vX,Y)\sim P_{\vX, Y}$, with $P_{\vX}$ the marginal distribution assumed to be known, and let $g(\vX) := \E [Y \mid \vX]$ denote the regression function. 

\begin{definition}[Mean dimension-reduction subspaces and central mean subspace~\cite{Cook02cms}]
  A \emph{mean dimension-reduction subspace} for $\E[Y \mid \vX]$ is a subspace $W \subseteq \R^d$ such that $\E[Y \mid \vX]$ is a function only of $\vV^\T \vX$ for some matrix $\vV$ whose columns form a basis for $W$.
  Define
  \begin{equation*}
    \minspace:= \bigcap \{ W \subseteq \R^d : \text{$W$ is a mean dimension-reduction subspace for $\E[Y \mid \vX]$} \}.
  \end{equation*}
  We say $\minspace$ is the \emph{central mean subspace (CMS)} for $\E[Y \mid \vX]$ if $\minspace$ is a mean dimension-reduction subspace.
  (If $\minspace$ is not a mean dimension-reduction subspace, then the CMS does not exist.)

\end{definition}

The CMS is known to exist under mild conditions on $\supp(\vX)$, as discussed by \cite{Cook02cms}; an example sufficient condition is that $\supp(\vX)$ is open and convex.
Henceforth, we shall assume that $\supp(\vX)$ satisfies a sufficient condition that guarantees the existence of the CMS.
(Such an assumption will be explicit given later in \Cref{assume:basic}.)

Let $k = \dim(\minspace)$, and let the columns of $\vU := [\vu_1, \ldots, \vu_k] \in \R^{d \times k}$ be an orthonormal basis for $\minspace$. 
Because $\minspace$ is the CMS for $\E[Y \mid \vX]$, the regression function has the form of a \emph{$k$-index model}, specifically
\begin{align}
    \label{eqn:model}
    g(\vX) =  f(\vU^{\T} \vX)
\end{align} 
for some function $f \colon \R^k \to \R$, which we call the \emph{link function}.
Our goal is to estimate $\minspace$ (or equivalently, $\col(\vU)$) from an i.i.d.~sample from $P_{\vX, Y}$.
We evaluate our estimate using the following measure of error.

\begin{definition}[Distance with Optimal Rotation, Adapted from \cite{Chen2021spectral}]
    \label{def:distance}
    Given an orthogonal matrix $\vU$ and its estimate $\widehat\vU$ (which is also orthogonal), the distance with optimal rotation between them defined as:
        $$d(\widehat\vU, \vU) := \min_{\vR \in \cO^{k \times k}} \lVert \widehat\vU \vR - \vU\rVert_{\op}\,, $$
    where $\cO^{k \times k}$ is the set of all $k \times k$ orthogonal matrices. 
\end{definition} 

The distance with optimal rotation in \Cref{def:distance} turns out to be equivalent to the so-called \emph{$\mathbf \sin_\Theta$ distance}, another popular notion of distance between subspaces given as follows. Let $\vU_{\perp}  \in \R^{d \times (d-k)}$ denote a matrix whose columns form an orthonormal basis for the orthogonal complement $\minspace^{\perp}$ of $\minspace$. Then, the $\sin_\Theta$ distance between $\vU$ and $\widehat{\vU}$ is defined as $\sin_\Theta(\widehat{\vU}, \vU) := \lVert  \widehat{\vU}^{\T} \vU_{\perp} \rVert_{\op}.$ 

It can be shown (Lemma 2.6 from \cite{Chen2021spectral}) that there exists $c, C > 0$ such that $$c \cdot d(\widehat{\vU}, \vU) \le \sin_\Theta (\widehat{\vU}, \vU) \le C \cdot d(\widehat{\vU}, \vU).$$

The distance with optimal rotation is appealing in that it readily yields control on regression error when data is first projected to $\col(\widehat{\vU})$ rather than $\minspace$. For instance, suppose $f$ is $L$-Lipschitz (e.g, as in linear regression, or as in polynomial regression with a bounded domain), and let $\tilde{f}:= f \circ \vR^{\T}$, where $\vR \in \cO^{k \times k}$ is such that $d(\widehat{\vU},\vU) = \lVert \widehat{\vU} \vR - \vU \rVert_{\op}$, then:
$$\lVert \tilde{f}(\widehat{\vU}^{\T} \vX) - f (\vU^{\T} \vX)  \rVert_{P_{\vX}} := \lVert f((\widehat{\vU}\vR)^{\T} \vX) - f (\vU^{\T} \vX)  \rVert_{P_{\vX}} \le  L \cdot \lVert \vX \rVert_{P_{\vX}} \cdot d(\widehat{\vU}, {\vU})\,.$$
To estimate $\tilde{f}$, one can apply the classical Nadaraya-Watson estimator on the projected sample of $(\widehat{\vU}^{\T} \vX, Y)$. If $d(\widehat{\vU}, \vU) \lesssim n^{-1/(k+2)}$, the regression error rate is the same as if $\vU$ is known.

\section{Main Results} 
  \label{sec:results}

We first introduce a surrogate object $\ESGOP$ in Section \ref{sec:ESGOP} below, and show that it indeed is \emph{exhaustive}, i.e., it recovers all of $\col({\bf U})$. Following this, we derive a separate surrogate object 
$\ASGOP$ in Section \ref{sec:estimation} which approximates $\ESGOP$ and turns out to be simpler to estimate and retains the same exhaustiveness properties. Our main algorithm then estimates $\ASGOP$, and returns the column space $\hat \vU$ of this estimate. Our main result of Theorem \ref{thm:generalrate} which bounds the estimation error is presented in Section \ref{sec:general-link}, followed by an instantiation for polynomial link functions in Section \ref{sec:poly}. 

We will need the following few assumptions. 

\begin{assumption} 
    \label{assume:basic}
    The joint distribution $P_{\vX,Y}$ satisfies the following basic conditions. 
    \begin{enumerate}[label={(\arabic*)}, ref={(\arabic*)}]
        \item \label{assume:basic:differentiable}The link function $f$ is differentiable, with gradient $\nabla f(\vz), \forall \vz \in \R^{k}$; 
        \item \label{assume:basic:support} The marginal $P_{\vX}$ admits a Lebesgue density $p$ supported on $\R^d$, i.e., $p> 0$ everywhere.
    \end{enumerate}
\end{assumption}

The assumption of differentiability above implies that of the regression function $g$, and is needed as our approach relies on gradient estimates. The second assumption above will be required in designing the estimation procedure, and also ensures the existence of $\minspace$ as previously remarked. 

\subsection{Expected Smoothed Gradient Outer Product} \label{sec:ESGOP}


We now present a smoothed version of the expected gradient outer product $\vM := \E_{\vX \sim P_{\vX}}[\nabla g(\vX) \nabla g(\vX)^{\T}]$, which was discussed in the introduction. 

\begin{definition}
\label{def:ESGOP}
The {Expected Smoothed Gradient Outer Product (ESGOP)} $\ESGOP$ is defined as follows.
First, let $h>0$ be given,
and define the \emph{smoothed gradient} of $g$ at $\shift \in \R^d$ as: 
\begin{align*}
    \vbeta_h(\shift) &:= \vU \cdot \Egrad{h}(\shift), \quad \text{where } \Egrad{h}(\shift):= \E_{\vZ \sim \Normal(\vzero_d, h^2 \vI_d)}[\nabla f(\vU^{\T}(\vZ+\shift))].  
\end{align*} 
Let $\sigma_\theta>0$. We then have: 
\begin{equation} \label{eqn:M}
    \ESGOP := \E_{\shift \sim \Normal(\vzero_d, \sigma_\theta^2 \vI_d)} \left[\vbeta_h(\shift) \vbeta_h(\shift)^{\T}\right] = \vU\left(\E_{\shift \sim \Normal(\vzero_d, \sigma_\theta^2 \vI_d)}\left[  \Egrad{h}(\shift)\Egrad{h}(\shift)^\T\right]\right) \vU^\T.
\end{equation}
\end{definition}

The ESGOP $\ESGOP$ differs from the EGOP $\vM$ in two important ways. First, the gradient $\nabla f(\vU^{\T} \vX)$ is replaced by a smoothed gradient $\Egrad{h}(\shift)$. Second, the expectation in the definition of $\ESGOP$ is taken over a Gaussian distribution (i.e., $\vtheta \sim \Normal(\vzero_d,\sigma_{\theta}^2 \vI_d )$), instead of the data distribution $P_{\vX}$ as in the EGOP. As we will see later in \Cref{sec:estimation}, such a design will allow us to invoke Stein's Lemma (\cref{lem:stein}) for Gaussians.

\subsubsection{Exhaustiveness of the ESGOP $\ESGOP$}
From \cref{eqn:M} it is clear that $\col(\ESGOP)$ is contained in $\minspace := \col(\vU)$. We argue here that the two column spaces are in fact equal. From the same equation, this boils down to arguing that 
$$ \E_{\shift \sim \Normal(\vzero_d, \sigma_\theta^2 \vI_d)}\left[  \Egrad{h}(\shift)\Egrad{h}(\shift)^\T\right] $$ 
is full-rank. To gain some intuition, let's first revisit a similar argument for the EGOP 
$\vM = \vU \left( \E_{\vX \sim P_{\vX}} \nabla f(\vU^{\T}\vX) \nabla f(\vU^{\T}\vX)^{\T}\right) \vU^{\T}. $

Suppose $\vM$ were not exhaustive, then one can find a direction $\vv \in \R^k$ with $\lVert \vv \rVert = 1$, such that $\E_{\vX \sim P_{\vX}}[(\vv^{\T}\nabla f(\vU^{\T}\vX))^2] = 0$. Equivalently, $\vv^{\T}\nabla f(\vU^{\T}\vx) = 0$ for $P_{\vX}$-almost every $\vx$. Then, $\vU \vv \in \R^d$ is a direction along which the directional derivative of $g$ is zero almost surely.  In other words, $\minspace$ admits an irrelevant direction for the regression function $g$, which contradicts the assumption that $\minspace$ is the CMS and hence cannot be further reduced. 

 The exhaustiveness of $\ESGOP$ is established in a similar way as above, with the added difficulty that we now have to instead argue about directional \emph{smoothed} derivatives $\vv^\T \Egrad{h}$ and how they might similarly lead to a contradiction on the minimality of $\minspace$. This is established by relating $\vv^{\T} \Egrad{h}$ to $\vv^{\T} \nabla f$ by a Hermite expansion of $f$. {Note that, now the contradiction on minimality of $\minspace$ can only be established $\Normal$-a.s., which fortunately implies $P_X$-a.s.\ since $P_X$ has a density and is therefore $\Normal$ dominated.} Such arguments lead to the following proposition whose proof is given in \Cref{sec:proof_exhaust}. The result requires an additional minor condition on square integrability of $f$ so that Hermite expansions are well defined.

\begin{proposition}[Exhaustiveness of $\ESGOP$]
  \label{prop:exhaust}
Suppose that \Cref{assume:basic} holds, and that the link function $f$ satisfies $\E_{\vZ\sim\Normal(\vzero_k, h^2 \vI_k)} [f(\vZ)^2] < \infty$. Then, for $h, \sigma_\theta > 0$, the ESGOP $\ESGOP$ is exhaustive, i.e., we have $\col({\ESGOP}) = \minspace.$
\end{proposition} 

\subsection{Average Smoothed Gradient Outer Product}
\label{sec:ASGOP} We now derive a first approximation to $\ESGOP$ by approximating the outer-expectation in \cref{eqn:M} from a random sample of $\vtheta$ locations. 

\begin{definition}
    \label{def:tildeM}
    Let $m, h > 0$, and let 
    $\shift_1, \ldots, \shift_m \overset{i.i.d.}{\sim} \Normal(\vzero_d, \sigma_{\theta}^2 \vI_d)$. The Average Smoothed Gradient Outer Product (ASGOP) is obtained as: 
    \begin{align}
    \label{eqn:asgop}
        \ASGOP 
        := \frac{1}{m} \sum_{j=1}^m \vbeta_h(\shift_j) \vbeta_h(\shift_j)^{\T} 
        = \vU\left(\frac{1}{m} \sum_{j=1}^m\left[  \Egrad{h}(\shift_j)\Egrad{h}(\shift_j)^\T\right]\right) \vU^\T.
    \end{align}
\end{definition}

We will require that $m$ is sufficiently large so that $\ASGOP$ inherits the exhaustiveness of $\ESGOP$. This is established in the following proposition via a concentration argument on eigenvalues (see \Cref{sec:proof_exhaust}). We require the following moment condition for matrix concentration. 

\begin{assumption}
\label{assume:smooth_gradient}
The smoothed gradient $\Egrad{h}$ satisfies:
    $$\deltamoment := \left(\E_{\shift \sim \Normal(\vzero_d, \sigma_\theta^2 \vI_d)}\left[\left\lVert\Egrad{h}\left({\shift}\right)\right\rVert^{4}\right] \right)^{1/4}< \infty.$$
\end{assumption}

Examples of link function $f$ for which the condition $\deltamoment < \infty$ is satisfied include Lipschitz functions (in which case $\nabla f$ is bounded and hence so is $\Egrad{h}$) and polynomials (see \Cref{lem:poly} (ii)).

\begin{corollary}[Exhaustiveness of $\ASGOP$]
    \label{cor:exhaust}
    Suppose that \Cref{assume:smooth_gradient} and the condition of \Cref{prop:exhaust} holds, and  let $0 < \delta < 1/2$, if $m \ge {8 \deltamoment^4} {\lambda_k(\ESGOP)^{-2}} \delta^{-1} (\log\left({4d}/{\delta}\right))^2 $, then with probability at least $1- \delta$, the ASGOP $\ASGOP$ satisfies $\col(\ASGOP) = \minspace.$
\end{corollary}

\begin{remark} 
\label{rmk:m_and_asgop}
The above corollary is important in that we can now focus on estimating the simpler object $\ASGOP$, which turns out to be exhaustive given just $m = O(1)$ locations  $\vtheta_j$'s. Such intuition leads to the estimator of the next section, which our main theorem relies upon  (\Cref{thm:generalrate}). Note that, estimating $\ESGOP$ itself appears to require more resources, namely, $m = O(n)$ (see \Cref{rmk:ESGOP}). 

\end{remark}

\subsection{Estimation of CMS}
\label{sec:estimation}
We now present a simple estimator $\estimator$ of the $\ASGOP$ from i.i.d. samples $\Data := \{ \vX_i, Y_i\}_{i =1}^n \sim P_{\vX, Y}$; let $\widehat \vU$ denote the top $k$ eigenvectors of $\estimator$, then $\col(\widehat{\vU})$ estimates $\minspace$. We assume access to the density $p$ of $P_{\vX}$. 

To this end, we first propose an unbiased estimator of the smoothed gradient $\vbeta_h(\shift)$ of the regression function $g$, by relying on the following Stein's Lemma: 

\begin{lemma}[Stein's Lemma, \cite{chen2011stein}]
    \label{lem:stein}
    Let $\vZ \sim \Normal(\shift, h^2 \vI_d)$ and suppose $g: \R^d \rightarrow \R$ is differentiable with $\E \lVert \nabla g(\vZ) \rVert < \infty $. We then have
    \begin{align*}
     \E [g(\vZ)(\vZ-\shift)] = h^2 \E [\nabla g (\vZ)]. 
    \end{align*}
\end{lemma}
\begin{algorithm}[t!]
\caption{Estimation of the CMS $\minspace$}\label{alg:main}
\setstretch{1.35} \small 
\begin{algorithmic}[1]
    \vspace{.5em}
    \State \textbf{Input:} {$h > 0, \sigma_\theta > 0, m \in \mathbb{Z}$, dataset $\Data \subset \R^d \times \R$} 
    \State Randomly sample $\{\shift_{j}\}_{j=1}^{m}$ from $\Normal(\vzero_d, \sigma_\theta^2 \vI_d )$ 
    \State Split $\Data$ into  $\Data_{1}$, $\Data_{2}$, \ldots, and $\Data_{m}$ with equal sizes

    \State Split each $\Data_{j}$ into $\Data_{j, 1}$ and $\Data_{j, 2}$ with equal sizes
    
    \State $\forall j \in [m]$, let $\hat{\vbeta}_{j,1} \gets \hat{\vbeta}_h(\shift_{j}, \Data_{j, 1})$, and  $\hat{\vbeta}_{j,2} \gets \hat{\vbeta}_h(\shift_{j}, \Data_{j,2})$  \Comment As defined in \Cref{eq:t_LR_rw_me}.
  
    \State $\estimator \gets \frac{1}{2m}\sum_{j=1}^{m}(\hat{\vbeta}_{j,1}\hat{\vbeta}^{\T}_{j,2}+\hat{\vbeta}_{j,2}\hat{\vbeta}^{\T}_{j,1})$ 
    \State \textbf{Return:} {$\widehat{\vU} \in \R^{d \times k}$, the top-$k$ eigenvectors of $\estimator$}

\end{algorithmic}
\end{algorithm}
For intuition, notice that we have $\vbeta_h(\shift) = \E_{\vZ\sim\Normal(\vzero, h^2 \vI_d)}\nabla g(\vZ + \shift) = \E_{\vZ\sim\Normal(\shift,h^2 \vI_d)}\nabla g(\vZ)$.
 
Thus, suppose for a moment we were to sample $\vX \sim \Normal(\shift, h^2 \vI_d)$, then we would have 
$$\E Y \cdot (\vX-\shift) = \E g(\vX)\cdot (\vX - \shift) = h^2 \vbeta_h(\shift), $$
that is $h^{-2} Y \cdot (\vX-\shift) $ would be unbiased for $\beta_h (\shift)$. Our main insight therefore is to proceed by importance-weighting, that is, reweighting the estimator as 
$\rho_h(\vX; \shift) \cdot Y \cdot (\vX-\shift)$ so as to recover expectation under $\Normal(\shift, h^2 \vI_d)$. In particular, recalling $p$ the density of $P_{\vX}$, we let 
\begin{align}
    \label{eqn:density_ratio}
    \rho_h(\vx; {\shift}):= \frac{\varphi_h(\vx - \shift)}{p(\vx)}\,, \quad \text{where } \varphi_h(\cdot) \text{ denotes the density of $\Normal(\vzero_d, h^2 \vI_d)$ }. 
\end{align}

\paragraph{Estimating $\vbeta_h(\shift_j)$'s} Consider the $m$ values $\{\shift_j\}_{j \in [m]}$ drawn from $\Normal(\vzero_d, \sigma_{\theta}^2 \vI_d)$, upon which $\ASGOP$ is defined (\Cref{def:tildeM}). In order to evaluate $\vbeta_h(\shift_j)$'s, assume (without loss of generality) that the sample size $n$ is divisible by $2m$, and partition the dataset $\Data$ into independent $\{\Data_{j,\ell}\}_{j \in [m], \ell = 1,2}$ of equal size. We then define the following estimators: 
\begin{equation}        
  \label{eq:t_LR_rw_me}
     \hat{\vbeta}_h(\shift_j; \Data_{j, \ell}) := \frac{h^{-2}}{|\Data_{j, \ell}|}\sum_{(\vX, Y) \in \Data_{j, \ell}} \rho_h(\vX; \shift_j) \cdot Y \cdot (\vX-\shift_j).
\end{equation}
Thus, for each $\vtheta_j$, we have two independent estimates $\hat \vbeta_{j, \ell} := \hat{\vbeta}_h(\shift_j; \Data_{j, \ell})$, $\ell = 1, 2$. 
Each such estimate is unbiased, following the above intuition, as formalized below. 

\begin{proposition}
     We have $\E \hat{\vbeta}_{j, \ell} = {\vbeta}_h(\shift_j)$ for each $j \in [m]$ and $\ell = 1, 2.$ 
\end{proposition}

{
\begin{remark}
    The estimators defined in \Cref{eq:t_LR_rw_me} depend on the knowledge of the density ratio $\rho_h$. For the cases of unknown $\rho_h$, one may consider a plug-in estimator where $\rho_h$ is replaced by its estimates $\hat{\rho}_h$ that is calculated with another unlabeled dataset $\{\widetilde{\vX}_1, \ldots, \widetilde{\vX}_N\}$. Such a strategy works best when one has access to abundant unlabeled data while the budget for labeling is limited. For a more detailed discussion of the effect of not knowing $\rho_h$, please refer to \Cref{sec:unknown_density_ratio}.
\end{remark}
}
\begin{remark}
    Note that $P_{\vX}$ being fully supported is a necessary condition for importance reweighting to recover exactly the expectation under Gaussian measure by sampling from $P_{\vX}$. If $P_{\vX}$ has a bounded support, one can still approximate the expectation by choosing the variance parameter $h$ small, see \Cref{sec:bounded_support} for a detailed discussion.
\end{remark}

Notice that, by the independence of the two estimates of $\vbeta_h(\shift_j)$, we also immediately have that  
$\E[\hat{\vbeta}_{j,1} \hat{\vbeta}_{j,2}^{\T} \mid \shift_j] = {\vbeta}_h(\shift_j){\vbeta}_h(\shift_j)^{\T}$.
This leads to the estimator of $\ASGOP$ given below: 
\begin{align*}
    \estimator := \frac{1}{2m} \sum_{j=1}^m \left[\hat{\vbeta}_{j,1} \hat{\vbeta}_{j,2}^{\T} + \hat{\vbeta}_{j,2} \hat{\vbeta}_{j,1}^{\T}\right]. 
\end{align*}
The form of $\estimator$ ensures that it is symmetric. The full procedure is detailed in \Cref{alg:main}. 

\begin{remark}[Connection to Local Linear Regression]
\label{rmk:llr}
     Finally, we remark that the estimator $\hat{\vbeta}_h(\shift)$ is tightly related to a \emph{local linear regression estimator} as they can be shown to match asymptotically (see \Cref{sec:lle}). 
\end{remark}

\subsection{Upper-bound on Estimation Error} 
  \label{sec:general-link}

We start with some further assumptions, some of which imply the various conditions laid out so far in \Cref{prop:exhaust} and \Cref{cor:exhaust}, as explained below. These various conditions are required for concentration arguments on $\hat{\vbeta}_h$ and $\estimator$. 

\begin{assumption} [Tail Conditions for Noise]
  \label{assume:design1}
   The noise has a subgaussian tail, i.e., $\exists C, \sigma_Y^2>0$  such that $\forall \vx \in \R^d$, {$\p\{|Y - \E[Y \mid \vX=\vx]| > t \mid \vX =\vx\} \leq C \cdot \exp(-t^2/(2\sigma_{Y}^2))$}. 
\end{assumption}

\begin{assumption}[Moment Condition on Density Ratio]
    \label{assume:design2}
    The density ratio $\rho_h(\vX; \vzero_d)$, as defined in \cref{eqn:density_ratio}, satisfies: 
    \begin{align*}
        \ratiomoment := \left(\E_{\vX\sim P_{\vX}}\left[\rho_h(\vX; \vzero_d)^5\right]\right)^{1/5} < \infty\,.
    \end{align*}
\end{assumption}
\begin{remark}
    \label{rmk:choice_of_h}
    \Cref{assume:design2} holds when $P_{\vX}$ has a heavier tail than $\Normal(\vzero_d, h^2\vI_d)$, since $\rho_h(\vX; \vzero_d)$ is then bounded. Thus, intuitively, lighter tail $P_{\vX}$ are admitted by smaller choices of $h$. This is illustrated in the following proposition. 
\end{remark}

\begin{proposition}[Example of $\ratiomoment$ and choice of  $h$]
    \label{prop:var}    
    Suppose $P_{\vX} = \Normal(\vzero_d, \sigma
    ^2\vI_d)$. Then,

    \[\mu_{\rho} =
    \begin{cases}
        \left(\frac{5h^8}{\sigma^{8}} - \frac{4h^{10}}{\sigma^{10}}\right)^{-d/10}  \,, & \text{ when } 0 < h < \sqrt{5}\sigma/2  \\
         \infty\,, &  \text{ when } h \ge \sqrt{5}\sigma/2
    \end{cases}.
    \]
\end{proposition}
    
Next, we require some additional moment conditions on the link function $f$. 

\begin{assumption} [Moment Condition on Link Function]
  \label{assume:link}
    The link function $f$ satisfies:
    \begin{align*}
        \linkmoment := \left(\E_{\vX \sim P_{\vX}}[f(\vU^{\T} \vX)^{6}]\right)^{1/6} < \infty\, .
    \end{align*} 
\end{assumption}

\begin{remark}
    We note that the condition of \Cref{prop:exhaust} (exhaustiveness of $\ESGOP$), i.e., $f$ is square integrable with respect to $\Normal(\vzero_k, h^2I_k)$, holds under \Cref{assume:design2} and \ref{assume:link}. 

\end{remark}

The following result establishes an upper-bound of order $O(n^{-1/2})$ on the estimation error for the CMS $\minspace$, and holds for general link functions $f$ satisfying the minor moment conditions discussed above. In particular, this main result allows nonparametric, i.e., highly complex link and regression functions $f$ and $g$, and yet admits fast convergence to $\minspace$. 

\begin{theorem} [Main Result]
  \label{thm:generalrate}
   Let $\Data$ be a random sample of size $n$ from $P_{\vX, Y}$, $h, \delta > 0$, ${16 \deltamoment^4} {\lambda_k(\ESGOP)^{-2}} \delta^{-1} (\log\left({8d}/{\delta}\right))^2 \le m \le n/(2d)$ and $0 < \sigma_{\theta} < {h}/\sqrt{20}$. Let  $\widehat{\vU}$ be the output from \Cref{alg:main} with inputs $(h, \sigma_\theta, m, \Data)$. 
   Under \Cref{assume:basic}-\ref{assume:link}, there exists some absolute constant $C_1>0$, such that with probability at least $1-\delta$,

    \begin{align} \label{eqn:UEstimation}
        &d(\widehat{\vU}, \vU) \le C_1 \Bigg[{C_{h, \sigma_{\theta},1} \cdot {d(\linkmoment + \sigma_Y)^2} } \cdot \frac{\sqrt{m}}{n} + {C_{h, \sigma_{\theta},2} \cdot {d^{1/2}(\linkmoment + \sigma_Y)}} \cdot \frac{1}{\sqrt{n}} \  
         \Bigg] \cdot \frac{\log({8d}/{{\delta}})}{\sqrt\delta}\,,
    \end{align}
    where $C_{h, \sigma_{\theta},1}$ and $C_{h, \sigma_{\theta},2}$ are given by:
    \begin{align}
    \label{eqn:constant_h_sigma}
        C_{h, \sigma_{\theta},1} &:= (1-20\sigma_{\theta}^2/h^2)^{-d/2} \cdot {\ratiomoment^{5/3}} \cdot (\lambda_k(\ESGOP))^{-1}, \text{ and } \nonumber \\
        C_{h, \sigma_{\theta},2} &:= (1-10\sigma_{\theta}^2/h^2)^{-d/2} \cdot \deltamoment \cdot \ratiomoment^{5/6} \cdot (\lambda_k(\ESGOP))^{-1}.
    \end{align}
\end{theorem}

\begin{remark}[Constants in Rate] \label{rmk:no_sqrt_m} 
    Since the term $\sqrt{m}/n \le 1/\sqrt{n}$, 
    \Cref{thm:generalrate} establishes that $\minspace$ can be estimated at a parametric rate $C_d \cdot n^{-1/2}$, for a constant 
    $$C_d \leq C \cdot \max\{ C_{h, \sigma_{\theta},1} \cdot d, \ C_{h, \sigma_{\theta},2}  \cdot d^{1/2} \}.$$ 
    This prefactor $C_d$ depends indirectly on {the condition number $\lambda_1/\lambda_k$} of $\ESGOP$ since the largest eigenvalue $\lambda_1(\ESGOP)$ could be bounded 
    by some of the other moment parameters appearing in $C_d$; for instance  
    it is easy to see that $\lambda_1(\ESGOP) 
    \leq \E_{\shift \sim \Normal(\vzero_d, \sigma_\theta^2 \vI_d)}\left\lVert\Egrad{h}\left({\shift}\right)\right\rVert^{2} \leq  \deltamoment^2 $. For a more explicit dependence, see 
    \Cref{rmk:largest_eigen} relating to the proof of \Cref{lem:beta-error}. 
    
    Clearly, the various moment parameters, i.e., $\linkmoment,\deltamoment, \ratiomoment$ appearing in $C_d$ depend tightly on $P_{X, Y}$, and most importantly on the behavior of the link function $f$. In \Cref{sec:poly} below we consider explicit polynomial conditions on $f$ along with the example of {Gaussian $P_{\vX}$}, and illustrate $C_d$ via explicit bounds on moment conditions. 
\end{remark} 

In the above result, while the lower-bound on $m$ was required for the exhaustiveness of $\ASGOP$ (\Cref{cor:exhaust}), the upper-bound $m\leq n/(2d)$ ensures that the datasets used to estimate smoothed gradients $\Egrad{h}(\shift_j)$ are not empty.  

\begin{remark}[Estimating $\ESGOP$] \label{rmk:ESGOP}
As it appears, estimation of $\ESGOP$ may require larger values of $m$ than required in the above main theorem for estimation of $\minspace$. For instance consider upper-bounding 
$$\lVert \estimator - \ESGOP \rVert_{\op} \le \lVert \estimator - \ASGOP \rVert_{\op} + \lVert \ASGOP - \ESGOP \rVert_{\op}.$$ The second term $\lVert \ASGOP - \ESGOP \rVert_{\op}$ would necessarily introduce an $O(1/\sqrt{m})$ term which is of order $O(n^{-1/2})$ only for large $m = \Omega(n)$.

\end{remark}

    \begin{remark}[Dependence on $\delta$]
        The dependence on the failure probability $\delta$ is of order $\tilde{O}(\delta^{-1/2})$ which may appear large in light of usual high-probability results. This is due to the fact that we make minimal distributional assumptions, which in particular result in lower order moment conditions on 
        the random matrices $\hat{\vbeta}_{j,1}\hat{\vbeta}_{j,2}^{\T}$, preventing exponential concentration guarantees. 
        Nonetheless, heavy tail approaches such as \emph{median-of-means} (see \Cref{sec:mom}) may be employed to improve such dependence on $\delta$. 
    \end{remark}

\subsection{Instantiation for Polynomial Links} 
  \label{sec:poly}

In this section, we instantiate our main \Cref{thm:generalrate} under the assumptions of a polynomial link $f$ and standard Gaussian $P_X$. Such assumptions have been considered extensively in the recent literature on CMS estimation (see e.g. \cite{dudeja2018learning,chen2020learning,damian2022neural}) and are formalized below.  

\begin{assumption} 
\label{assume:poly}
We consider the following conditions on $P_{X,Y}$: 
\begin{enumerate}[label=(\arabic*)]
\item The marginal density $P_X$ is $\Normal(\vzero_d, \vI_d)$\,;
\item The link function $f$ is a polynomial of degree at most $r$. Furthermore, we assume w.l.o.g that $f$ is normalized, i.e., $\E_{\vZ \sim \Normal(\vzero_k, \vI_k)}[f(\vZ)^{2}] = 1$\,;
\item \Cref{assume:design1} holds, namely the noise is subgaussian with parameter $\sigma_Y^2 > 0$\,.
\end{enumerate}
\end{assumption}

\Cref{assume:poly} above supersedes all previous \Cref{assume:basic}-\ref{assume:link}
as will be evident in the proof of \Cref{cor:poly-rate} ({see Section~\ref{sec:proof_poly}}). 
In particular, as per \Cref{rmk:no_sqrt_m}, we are interested in understanding the constant $C_d$ in the rate of \Cref{thm:generalrate}, by explicitly bounding all the various moment parameters $\linkmoment, \ratiomoment, \deltamoment$ involved.
 
We adopt the following notion of minimum signal strength from \cite{dudeja2018learning}, which will serve to characterize $\lambda_k(\ESGOP)$.

\begin{definition} [Minimum Signal Strength]
  \label{def:minimum-signal-strength}
  Let $\tau>0$ be define as: $$\tau := \min_{\veta \in \mathbb{S}^{k-1}} \ \E_{\vZ\sim\Normal(\vzero_k, \vI_k)}\left[\left(\veta^{\T} \nabla f(\vZ)\right)^{2}\right]\,,$$
  where $\mathbb{S}^{k-1} \doteq \{\vx \in \R^k: \lVert \vx \rVert = 1\}$ is the sphere of the $k$-dimensional unit ball. 
\end{definition}

\begin{remark}
    The fact that $\tau \neq 0$ is guaranteed by the minimality of $\minspace$, since otherwise, for some $\veta \in \mathbb{S}^{k-1}$, $\vU \veta \in \minspace$ would be an irrelevant direction for the regression function $g$.  
\end{remark}

The following corollary establishes an upper-bound on subspace estimation error, with an explicit dependence on $d, h$ and $\sigma_{\theta}$.

\begin{corollary} 
  \label{cor:poly-rate}
 Suppose that \Cref{assume:poly} holds. 
 
 Let $\widehat{\vU}$ be the output from \Cref{alg:main} with inputs $0 < h < \sqrt{5}/2$, $0 < \sigma_{\theta} < h /\sqrt{20}$ and $m = n/(2d)$. Then, there exists $C_2 > 0$, independent of $n$, $d$, $\delta$, $h$ and $\sigma_{\theta}$, such that with probability at least $1 - \delta$, 
\begin{equation*}
   d(\widehat{\vU}, \vU) \le C_2  \cdot A_{\sigma_{\theta}^2/h^2, d,r}\cdot B_{h,d,r}  \cdot \frac{d}{\tau\sqrt{n}} \cdot \frac{\log({8d}/{{\delta}})}{\sqrt\delta}\,.
\end{equation*}
where $A_{\sigma_{\theta}^2/h^2, d,r}, B_{h,d,r}$ are given as:
\begin{align*}
    A_{\sigma_{\theta}^2/h^2, d,r}&:= {\left(\sigma_{\theta}^2/h^2\right)^{-(r-1)}\cdot(1-20\sigma_{\theta}^2/h^2)^{-d/2}}\,; \\
    B_{h, d, r}&:=  \max\{h^{-2r+1},h^{r-1}\} \cdot {(5h^{8} - 4h^{10})^{-d/6}}\,. 
\end{align*}
\end{corollary}

\begin{remark}[Choice of $h, \sigma_{\theta}$]
\label{rmk:choice_of_h_sigma}
    In order to minimize the prefactor $C_2  \cdot A_{\sigma_{\theta}^2/h^2, d,r}\cdot B_{h,d,r}$, we can minimize $A_{\sigma_{\theta}^2/h^2, d,r}$ (with respect to $\sigma_{\theta}^2/h^2$) and $B_{h,d,r}$ (with respect to $h$) separately.
    These terms are minimized at $\sigma^2_{\theta}/h^2 = (r-1)/(20(r-1)+10d)$ and $h=1$. 
    With such optimal choices, we have $A_{\sigma_{\theta}^2/h^2, d,r} \le (20 + 30d/(r-1))^{r-1}$, and $B_{h,d,r} = 1$. Hence $d(\widehat{\vU}, \vU) \le C_d \cdot n^{-1/2}$ where $C_d \propto d^r$, ignoring the logarithmic term and $\delta$. 
\end{remark}

\begin{remark}[Bound Optimality]
The dependence on $d$ in the bound of \Cref{cor:poly-rate} is hard to improve since our estimator falls in the class of so-called \emph{Correlational Statistical Query (CSQ)} learners, i.e., learners that rely solely on statistics of form $\sum_i Y_i \cdot \Phi(\vX_i)$ for some function $\Phi$ (see \cite{BENDAVID1995240,kearns98csq,diakonikolas20d} for an exact definition). Such CSQ learners have been shown in the recent work of \cite{damian2022neural} to require a sample complexity $n=\Omega(d^{r/2})$ to achieve nontrivial error in subspace estimation with bounded adversarial noises. As a heuristic comparison, our method requires $n = O(d ^{2r})$, which corresponds to a constant factor mismatch in the exponent of $d$.  A recent work \cite{damian2023smoothing} shows that such a gap can be closed for single-index models with the knowledge of the link function. However, for general multi-index models, it remains unclear how such a gap may be closed for CSQ learners as the result of \cite{damian2022neural} also leaves this open. 

We note however that, outside of CSQ procedures, much more benign sample size requirements are possible, for instance $n= O(d)$ was shown for a recent method of \cite{chen2020learning}.

\end{remark}

\section{Analysis Overview}
    \label{sec:analysis}
In this section, we outline the proofs of results in \Cref{sec:results}. Some further details can be found in \Cref{sec:omitted_proof}.

\subsection{Proof of \Cref{prop:exhaust} (Exhaustiveness of $\ESGOP$)}
    \label{sec:proof_exhaust}
\begin{proof}
    We first consider the case $h=1$. Using the Hermite expansion of $f$:
$
    f\left(\vz\right) = \sum_{\valpha\in\N^{k}}\hermite_{\valpha}H_{\valpha}\left(\vz\right) 
$, we have
\begin{align} 
\label{eq:delta-i}
\E_{\vZ \sim \Normal(\vzero, \vI_d)}[\partial_i f(\vU^{\T}(\vZ+\vtheta))] 
   &= \sum_{\valpha\in\N^{k}: \alpha_i \ge 1}\hermite_{\valpha}\left(\prod_{j\neq i}\frac{1}{\sqrt{\alpha_{j}!}}(\vU_{j}^{\T} \vtheta)^{\alpha_{j}}\right)\frac{\sqrt{\alpha_{i}}}{\sqrt{\left(\alpha_{i}-1\right)!}}(\vU^{\T}_i\vtheta) ^{\alpha_{i}-1} \nonumber \\
  \overset{\alpha_i = \alpha_i - 1}&{=}   \sum_{\valpha\in\N^{k}}\sqrt{\alpha_{i}+1}\cdot\hermite_{\valpha{(i)}} \prod_{j \in [k]} \left(({\vU_{j}^{\T} \vtheta)^{\alpha_j}}/{\sqrt{\alpha_j!}}\right) , \nonumber \\
   &=: \sum_{\valpha\in\N^{k}} \lambda_{\valpha,i} S_{\valpha}(\vU^{\T}\vtheta)\,,
\end{align}
where $\valpha{(i)} := (\alpha_{1},\ldots,\alpha_{i-1},\alpha_{i}+1,\alpha_{i+1},\ldots,\alpha_{k})\in\R^{k}$, $\lambda_{\valpha,i}:= \sqrt{\alpha_{i}+1}\cdot\hermite_{\valpha{(i)}} \in \R$ and $S_{\valpha}(\vz) = \prod_{j \in [k]} ({z_j^{\alpha_j}}/{\sqrt{\alpha_j!}})$ is a scaled monomial basis indexed by $\valpha$.

According to the definition, $\vM = \vU \vLambda \vU^{\T}$, where $\vLambda := \E_{\shift}[\Egrad{h}(\shift)\Egrad{h}(\shift)^{\T}] \in \R^{k\times k}$. To show $\col({\vM}) = \col(\vU)$, we only need to show that the matrix $\vLambda$ is full-rank.  Suppose not, then there exists a unit vector $\veta \in\R^{k}$ such that $\veta^{\T}\vLambda\veta = 0$, which indicates $\veta^{\T}\Egrad{h}(\shift) = 0$, for almost every $\shift$ with respect to Gaussian measure. Therefore, 
\begin{align}
\label{eqn:zero_coeff} 
\forall \valpha \in \N^k, \quad \sum_{i \in [k]} \eta_{i} \lambda_{\valpha,i} = 0\,. 
\end{align}
Next, we show that \cref{eqn:zero_coeff} implies that $\veta^{\T} \nabla f(\vU^{\T}\vx) = 0$ for almost every $\vx \in \R^d$. For each $i \in [k]$,
\begin{align}
  \label{eq:partial}
 \partial_i f(\vU^{\T} \vx)
  & = \sum_{\valpha\in\N^{k}}\hermite_{\valpha} \cdot {\partial_i H_{\valpha}(\vU^{\T}\vx)} \nonumber \\
  & = \sum_{\valpha\in\N^{k}:\alpha_{i}\geq 1}\hermite_{\valpha} \cdot (\sqrt{\alpha_{i}}H_{\alpha_{i}-1}(\vU_{i}^{\T}\vx)) 
 \prod_{j\neq i}H_{\alpha_{j}}(\vU_{j}^{\T}\vx)\nonumber \\
  \overset{\alpha_i = \alpha_i - 1}&{=}  \sum_{\valpha\in\N^{k}} \sqrt{\alpha_i+1} 
 \cdot \hermite_{\valpha{(i)}} \cdot H_{\valpha}(\vU^{\T}\vx)  \nonumber \\
  & = \sum_{\valpha\in\N^{k}} \lambda_{\valpha,i} \cdot H_{\valpha}(\vU^{\T}\vx) \,.
\end{align}
Therefore, we have for almost every $\vx \in \R^d$ with respect to the Gaussian measures, and hence $P_{\vX}$-a.s. (since $P_{\vX}$ is dominated by Gaussian), 
\begin{align*}
    \veta^{\T} \nabla f(\vU^{\T} \vx) \overset{\cref{eq:partial}}{=} \sum_{\valpha \in \N^k} \left(\sum_{i \in [k]} \eta_i \lambda_{\valpha,i}\right) H_{\valpha}(\vU^{\T}\vx) \overset{\cref{eqn:zero_coeff}}{=} 0\,.
\end{align*}
Then, it follows that
$$(\vU\veta)^{\T} \nabla g(\vx) = \veta^{\T} \vU^{\T}\vU \nabla f(\vU^{\T}\vx) = \veta^{\T} \nabla f(\vU^{\T}\vx) = 0$$
In other words, $\vU \veta \in \minspace$ is a direction along which the directional derivative of $g$ is zero almost surely.  In other words, $\minspace$ admits an irrelevant direction for the regression function $g$, which contradicts the assumption that $\minspace$ is the CMS and hence cannot be further reduced. 

For the case $h \neq 1$, one can let ${f}_h(\vz) := f(h\vz)$, then $\E_{{\vZ}\sim\Normal(\vzero, \vI_d)}[f_h(\vZ)^2] = 1$, and
$$\Egrad{h}(\shift) := \E_{\vZ\sim\Normal(\vzero, h^2 \vI_d)}[\nabla f(\vU^{\T}(\vZ+\shift))] = h^{-1} \E_{\vZ\sim\Normal(\vzero, \vI_d)}[\nabla f_h(\vU^{\T}(\vZ+\shift)))] = h^{-1} \overline{\nabla}_{1}f_h(\shift). \,$$
Clearly, $\E_{\shift} [\overline{\nabla}_{1}\tilde{f}(\shift)\overline{\nabla}_{1}\tilde{f}(\shift)^{\T}]$ is full-rank by the reasoning for the case $h=1$, and it follows that $\ESGOP = h^{-2} \vU \E_{\shift} [\overline{\nabla}_{1}\tilde{f}(\shift)\overline{\nabla}_{1}\tilde{f}(\shift)^{\T}] \vU^{\T}$ is exhaustive.
\end{proof}

\subsection{Proof of \Cref{cor:exhaust} (Exhaustiveness of $\ASGOP$)}

The following lemma (see \Cref{sec:proof_ASGOP} for the proof) establishes the concentration of the ASGOP $\ASGOP$ around ESGOP $\ESGOP$, and is the key ingredient for proving the exhaustiveness of $\ASGOP$. 
\begin{lemma} 
    \label{lemma:M_tilde}
        Under the setting of Theorem~\ref{thm:generalrate}, with probability $1-\delta$,
            $$\lVert \ASGOP - \ESGOP \rVert_{\op} \le  \left({\sqrt{2} \deltamoment^2} \cdot \frac{1}{\sqrt{m}}\right) \cdot \frac{1}{\sqrt{{\delta}}} \cdot \log\left(\frac{4d}{{\delta}}\right)\,,$$
        for all $0 < \delta < 1$.
\end{lemma}

Now we proceed to the proof of the \Cref{cor:exhaust}. 
\begin{proof}[Proof of \Cref{cor:exhaust} (Exhaustiveness of $\ASGOP$)]
     By \Cref{lemma:M_tilde}, we have with probability at least $1-\delta$, 
         $$\lVert \ASGOP - \ESGOP \rVert_{\op} \le {\lambda_k(\ESGOP)}/{2}\,,$$
    when $m \ge {8 \deltamoment^4} {\lambda_k(\ESGOP)^{-2}} \delta^{-1} (\log\left({4d}/{\delta}\right))^2$. Consequently, $\lambda_k(\ASGOP) \ge \lambda_k(\ESGOP)/2 > 0$ by Weyl's inequality. Therefore, $\col(\ASGOP)$ is of rank $k$, and hence equals to $\minspace$.
\end{proof}

\subsection{Proof of \Cref{thm:generalrate} (Main Theorem)}

The following Lemma~\ref{lemma:hat_M} (see \Cref{sec:proof_estimator} for the proof) establishes a probabilistic bound on $\lVert \estimator - \ASGOP \rVert_{\op}$. It is a key step for the proof of the main theorem.  
\begin{lemma}
    \label{lemma:hat_M}
    Under the setting of Theorem~\ref{thm:generalrate}, for all $0 < \delta < 1$ and some absolute constant $C > 0$, with probability $1-\delta$,
        \begin{equation}
            \label{eqn:M_error}
            \lVert \estimator - \ASGOP \rVert_{\op} \le C \left(\tilde{C}_{h,\sigma_{\theta},1} \cdot  d (\linkmoment + \sigma_Y)^2  \cdot \frac{\sqrt{m}}{n} + \tilde{C}_{h,\sigma_{\theta},2} \cdot  d^{1/2} (\linkmoment + \sigma_Y)  \cdot \frac{1}{\sqrt{n}}\right) \frac{\log({4d}/{{\delta}})}{\sqrt{{\delta}}}  \,.
        \end{equation}
     where $\tilde{C}_{h, \sigma_{\theta}, \ell}$, $\tilde{C}_{h, \sigma_{\theta}, 2}$ are as defined as:
    \begin{align*}
      \tilde{C}_{h, \sigma_{\theta}, 1} := (1-20\sigma_{\theta}^2/h^2)^{-d/2} \cdot \ratiomoment^{5/3}; \quad 
      \tilde{C}_{h, \sigma_{\theta}, 2} := (1-10\sigma_{\theta}^2/h^2)^{-d/2} \cdot \deltamoment \cdot \ratiomoment^{5/6}.
  \end{align*}
\end{lemma}

Now we proceed to the proof of the main theorem. 
\begin{proof} [Proof of \Cref{thm:generalrate} (Main Theorem)]
 Let $\tilde{\delta} = \delta/2 \in (0,1/2)$. We define two good events with desired properties for the concentration of the subspace estimator:  Let $\xi_1$ denote the event where $\minspace$ can be fully recovered from $\ASGOP$, i.e., $\col(\ASGOP) = \minspace$. Since $m \ge {8 \deltamoment^4} {\lambda_k(\ESGOP)^{-2}} \tilde{\delta}^{-1} (\log({4d}/{\tilde\delta}))^2$, we have the guarantee that $\p(\xi_1) \ge 1-\tilde\delta$ by \Cref{cor:exhaust}.
            Let $\xi_2$ denote the good event in \Cref{lemma:hat_M} with $\delta = \tilde\delta$.
        Next, we show that \cref{eqn:UEstimation} holds on $\xi_1 \cap \xi_2$, whose probability is at least $1-\delta$.  This follows from a standard application of the \emph{Davis-Kahan Theorem} (Theorem V.3.6 from \cite{stewart90perturb}). In particular, the eigengap can be calculated as:
        \begin{align*}
            \lambda_{k}(\ASGOP) - \lambda_{k+1}(\estimator) &= \lambda_{k}(\ASGOP) - (\lambda_{k+1}(\estimator) - \lambda_{k+1}(\ASGOP)) \\
            & \ge \lambda_k(\ASGOP) - \lVert \estimator - \ASGOP \rVert_{\op} \\
            &\ge \lambda_k(\ASGOP)/2 \,,
        \end{align*}
        for large $n$ that ensures $\lVert \estimator - \ASGOP \rVert_{\op} \le \lambda_k(\ASGOP)/2$. Therefore,
        $$ d(\widehat{\vU}, \vU) \le \frac{\sqrt{2} \lVert \estimator - \ASGOP \rVert_{\op}}{\lambda_{k}(\ASGOP) - \lambda_{k+1}(\estimator)} \le \frac{2 \sqrt{2} \lVert \estimator - \ASGOP \rVert_{\op}}{\lambda_k(\ASGOP)} \le \frac{4 \sqrt{2} \lVert \estimator - \ASGOP \rVert_{\op}}{\lambda_k(\ESGOP)}.$$
        By substituting \cref{eqn:M_error} into the last inequality, we conclude the proof.
    \end{proof}

\subsection{Proof of \Cref{cor:poly-rate}}
  \label{sec:proof_poly}
  To invoke the main \Cref{thm:generalrate} for the special case of polynomial links, we rely on the following lemma (see \Cref{sec:proof:lemma_poly} for proof) that quantifies the various moment parameters and $\lambda_k(\ESGOP)$ in terms of $d, r,  \tau, h$ and $\sigma_{\theta}$.

\begin{lemma}	
  \label{lem:poly}
Under the conditions for \Cref{cor:poly-rate}, we have:
\begin{enumerate}[label=(\roman*)]
    \item $\ratiomoment = (5h^{8} - 4h^{10})^{-d/10}\,;$
    \item $\linkmoment \leq 5^{r/2}\,;$
    \item $\deltamoment \leq (r+1)^{(6k+r+2)/4} \cdot \max\{h^{r-1}, h^{-1}\}\,;$
    \item $\lambda_k\left(\ESGOP \right) \geq r^{-r-4k}  \cdot 2^{-3(r-1)-k} \cdot ((r-1)!)^{-1} \cdot \left(\sigma_{\theta}^2/h^2\right)^{r-1} \cdot \min\{h^{2(r-1)},1\} \cdot \tau\,.$
\end{enumerate}
\end{lemma}

Next, we proceed to prove \Cref{cor:poly-rate}, which merely substitutes the bounds from \Cref{lem:poly} into  \Cref{thm:generalrate}.

\begin{proof}[Proof of \Cref{cor:poly-rate}]
    By \cref{eqn:constant_h_sigma} and \Cref{lem:poly}, we have:
    \begin{align*}
        C_{h, \sigma_{\theta},1} &\le C_{k,r}  \cdot \max\{h^{-2r+2},1\}  {(5h^{8} - 4h^{10})^{-d/6}} \left(\sigma_{\theta}^2/h^2\right)^{-(r-1)}  (1-20\sigma_{\theta}^2/h^2)^{-d/2}\cdot \tau^{-1}\,; \nonumber \\
        C_{h, \sigma_{\theta},2} &\le C_{k,r} \cdot \max\{h^{-2r+1},h^{r-1}\}   {(5h^{8} - 4h^{10})^{-d/12}}\left(\sigma_{\theta}^2/h^2\right)^{-(r-1)} (1-10\sigma_{\theta}^2/h^2)^{-d/2}\cdot \tau^{-1} \,,  
    \end{align*}
    where $C_{k,r}$ is a constant depends only on $k$ and $r$. Thus, we have 
    $
        \max\{d \cdot C_{h, \sigma_{\theta},1}, d^{1/2} \cdot C_{h, \sigma_{\theta},2}\} 
         \le C_{k,r} \cdot A_{\sigma_{\theta}^2/h^2, d,r}\cdot B_{h,d,r} \cdot d
    $.
    Finally, we invoke \Cref{thm:generalrate} and get:
    \begin{align*}
        d(\widehat{\vU}, \vU) \le C_2 \cdot A_{\sigma_{\theta}^2/h^2, d,r}\cdot B_{h,d,r} \cdot \frac{d}{\tau\sqrt{n}} \cdot \frac{1}{\tau} \cdot \frac{\log({8d}/{{\delta}})}{\sqrt\delta}\,,
    \end{align*}
    where $C_2 := C_1 \cdot C_{k,r} \cdot \max\{\linkmoment + \sigma_Y, (\linkmoment + \sigma_Y)^2\}$.
\end{proof}

\section{Simulations} 
  \label{sec:experiment}

In this section, we demonstrate how to choose the hyperparameters $h$ (the ``radius" of smoothing) and $m$ (the number of $\shift_j$ sampled in ASGOP) properly through a simple example with polynomial link function and standard Gaussian design. The related theoretical discussions can be found in \Cref{rmk:m_and_asgop}, \ref{rmk:choice_of_h}, \ref{rmk:ESGOP} and \ref{rmk:choice_of_h_sigma}.

\paragraph{Experiment Settings} Throughout the experiments, we consider a multi-index model~\cref{eqn:model} with $d=10,  k=3$, and
$
    \vU = [\ve_{1}, \ldots, \ve_{k}] \in \R^{d \times k},
$
where $\ve_{i}$ denotes the $i$-th column of the $d \times d$ identity matrix. That is, only the first three coordinates of $\vX$ are relevant directions for $Y$. The link function $f:\R^3 \to \R$ is a polynomial: $f(\vz) = z_{1}^{2} + z_{2}z_{3}$.

\paragraph{Choices of $h$ under Different $P_{\vX}$} Recall from \Cref{rmk:choice_of_h} that our method work for any $P_{\vX}$ with heavier tails than Gaussian. The choice of $h$ is critical when $P_{\vX}$ is a Gaussian distribution (see \Cref{prop:var}). We verify the statement by applying \Cref{alg:main} with different choices of $h > 0$ and $\sigma_{\theta} = h/\sqrt{20 + 10d}$ (see Remark~\ref{rmk:choice_of_h_sigma} for the reasoning of such choice) to datasets generated from $P_{\vX}$ being either standard Gaussian or standard Cauchy. Under appropriate choices of $h$, our method works reasonably well for both cases.

\begin{figure}[hbp!]
    \centering
    \includegraphics[scale=0.45]{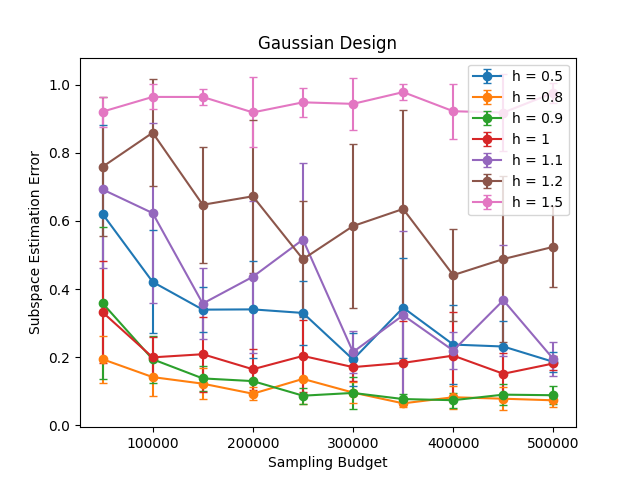}
    \includegraphics[scale=0.45]{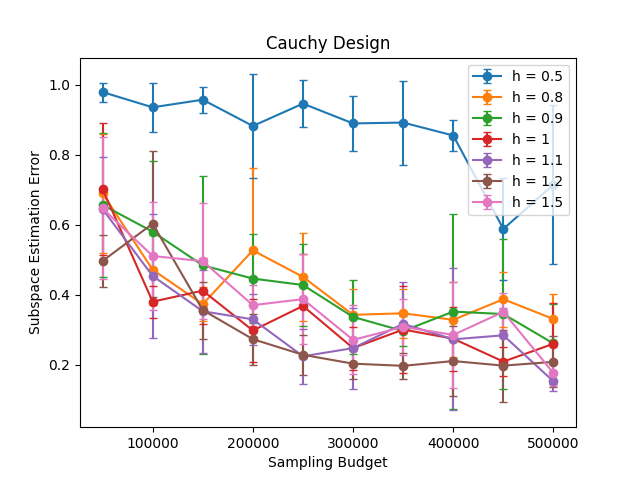}
    \caption{The subspace estimation error $d(\widehat{\vU}, \vU)$ v.s. sampling budget $n$. Here, we fixed the number of partitions $m = 15$ and $\sigma_{\theta} = h/\sqrt{20 + 10d}$. We replicate 10 times for each pair of $(n,h)$ and plot the mean (the dots) and the standard error (the error bars) of the estimation error. When $P_{\vX} =$ standard Gaussian (left), the performance of \Cref{alg:main} is quite sensitive to the choice of $h$. The optimal choice of $h$ is around the data variance 1. When $h$ gets smaller (e.g., $h=0.5$) or larger (e.g., $h=1.2, 1.5$), we observe larger errors under the same budget level. This actually coincides approximately with the minimizer of $\ratiomoment$ in terms of $h$ (c.f. \Cref{prop:var}). When $P_{\vX} =$ standard Cauchy, the method is more robust to the choice of $h$. }
    \label{fig:x_distribution}
\end{figure}
\paragraph{Choice of the Number of Partitions $m$ under Gaussian design} The number of partition $m$ affects the subspace estimation error rates in two ways. Firstly, when $m$ is not large enough, there is no guarantee that $\ASGOP$ is exhaustive (\Cref{cor:exhaust}). With a positive probability, one can only recover a proper subspace of the CMS $\minspace$, hence the subspace estimation error should be large. When $m$ gets larger, we expect the estimation error to grow in the order of $\sqrt{m}$ for a fixed $n$ according to the upper-bound in \Cref{thm:generalrate}. We evaluate the performance of \Cref{alg:main} under different choices of $m$, with a fixed sampling budget $n = 100,000$ under a standard Gaussian design. The related results are presented in \Cref{fig:choice_of_m}. 

\begin{figure}[ht!]
    \centering
    \includegraphics[scale=0.55]{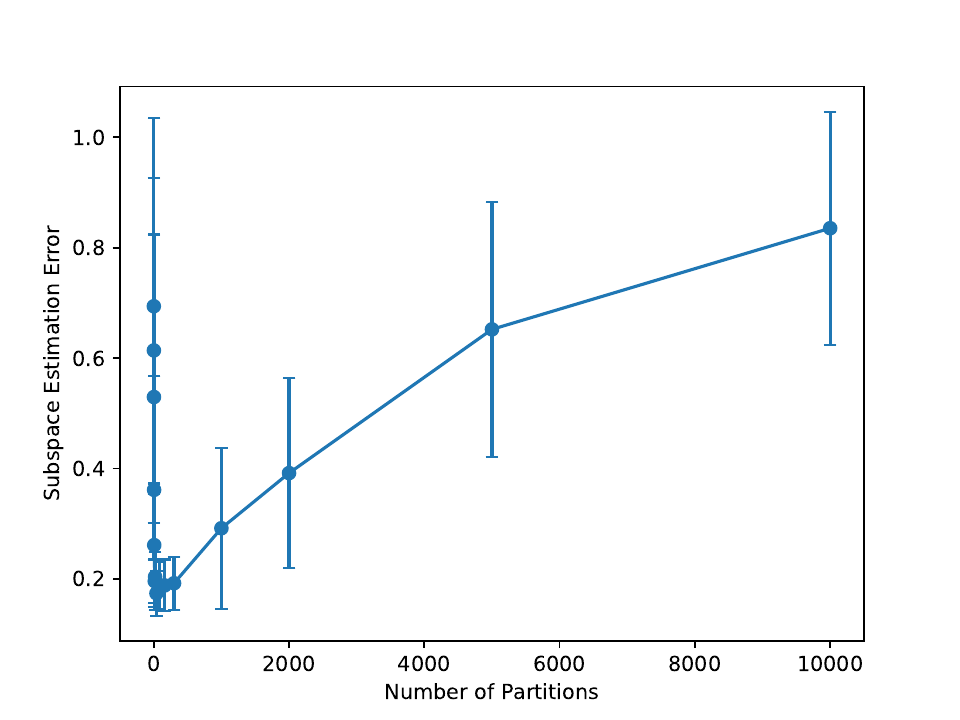}
    \caption{The subspace estimation error $d(\widehat{\vU},\vU)$ v.s. the choice of $m$. We replicate 10 times for each $m$, and plot the mean (the dots) and the standard error (the error bars) of the subspace estimation errors. When $m$ is small, the ASGOP $\ASGOP$ is not guaranteed to be exhaustive, and only a proper subspace of $\minspace$ can be recovered. This results in a large subspace estimation error. When $m$ is larger than a certain threshold (c.f., \Cref{cor:exhaust}), the ASGOP $\ASGOP$ is exhaustive with high probability, and the subspace error has an upper-bound that grows at the rate of $O(\sqrt{m})$. The result in the figure matches the reasoning above, as the subspace estimation error drops sharply atk the regime where $m$ is small, and grows gradually for large $m$.}
    \label{fig:choice_of_m}
\end{figure}

\newpage
\bibliographystyle{siamplain}
\bibliography{references}

\begin{thebibliography}{10}

\bibitem{babichev18slice}
{\sc D.~Babichev and F.~Bach}, {\em {Slice inverse regression with score
  functions}}, Electronic Journal of Statistics, 12 (2018), pp.~1507 -- 1543,
  \url{https://doi.org/10.1214/18-EJS1428},
  \url{https://doi.org/10.1214/18-EJS1428}.

\bibitem{BENDAVID1995240}
{\sc S.~Bendavid, A.~Itai, and E.~Kushilevitz}, {\em Learning by distances},
  Information and Computation, 117 (1995), pp.~240--250,
  \url{https://doi.org/https://doi.org/10.1006/inco.1995.1042},
  \url{https://www.sciencedirect.com/science/article/pii/S0890540185710425}.

\bibitem{bietti22learning}
{\sc A.~Bietti, J.~Bruna, C.~Sanford, and M.~J. Song}, {\em Learning
  single-index models with shallow neural networks}, in Advances in Neural
  Information Processing Systems, S.~Koyejo, S.~Mohamed, A.~Agarwal,
  D.~Belgrave, K.~Cho, and A.~Oh, eds., vol.~35, Curran Associates, Inc., 2022,
  pp.~9768--9783,
  \url{https://proceedings.neurips.cc/paper_files/paper/2022/file/3fb6c52aeb11e09053c16eabee74dd7b-Paper-Conference.pdf}.

\bibitem{Brillinger82glm}
{\sc D.~R. Brillinger}, {\em A generalized linear model with 'gaussian'
  regressor variables}, A Festschrift for Erich L. Lehmann, K. D. P.J. Bickel
  and J. Hodges, eds.,  (1982).

\bibitem{chen2011stein}
{\sc L.~H. Chen, L.~Goldstein, and Q.-M. Shao}, {\em Normal Approximation by
  Stein’s Method}, Springer, Berlin, Heidelberg, 2011,
  \url{https://doi.org/10.1007/978-3-642-15007-4}.

\bibitem{chen2020learning}
{\sc S.~Chen and R.~Meka}, {\em Learning polynomials in few relevant
  dimensions}, in Conference on Learning Theory, PMLR, 2020, pp.~1161--1227.

\bibitem{Chen2021spectral}
{\sc Y.~Chen, Y.~Chi, J.~Fan, and C.~Ma}, {\em Spectral methods for data
  science: A statistical perspective}, Foundations and Trends{\textregistered}
  in Machine Learning, 14 (2021), pp.~566--806,
  \url{https://doi.org/10.1561/2200000079},
  \url{https://doi.org/10.1561%2F2200000079}.

\bibitem{Cook02cms}
{\sc R.~Cook and B.~Li}, {\em {Dimension reduction for conditional mean in
  regression}}, The Annals of Statistics, 30 (2002), pp.~455 -- 474,
  \url{https://doi.org/10.1214/aos/1021379861},
  \url{https://doi.org/10.1214/aos/1021379861}.

\bibitem{cook1994interpretation}
{\sc R.~D. Cook}, {\em On the interpretation of regression plots}, Journal of
  the American Statistical Association, 89 (1994), pp.~177--189.

\bibitem{cook00save}
{\sc R.~D. Cook}, {\em {SAVE}: a method for dimension reduction and graphics in
  regression}, Communications in Statistics - Theory and Methods, 29 (2000),
  pp.~2109--2121, \url{https://doi.org/10.1080/03610920008832598},
  \url{https://doi.org/10.1080/03610920008832598},
  \url{https://arxiv.org/abs/https://doi.org/10.1080/03610920008832598}.

\bibitem{cook2009regression}
{\sc R.~D. Cook}, {\em Regression graphics: Ideas for studying regressions
  through graphics}, John Wiley \& Sons, 2009.

\bibitem{Dalalyan08}
{\sc A.~S. Dalalyan, A.~Juditsky, and V.~Spokoiny}, {\em A new algorithm for
  estimating the effective dimension-reduction subspace}, Journal of Machine
  Learning Research, 9 (2008), pp.~1647--1678,
  \url{http://jmlr.org/papers/v9/dalalyan08a.html}.

\bibitem{damian2022neural}
{\sc A.~Damian, J.~Lee, and M.~Soltanolkotabi}, {\em Neural networks can learn
  representations with gradient descent}, in Conference on Learning Theory,
  PMLR, 2022, pp.~5413--5452.

\bibitem{damian2023smoothing}
{\sc A.~Damian, E.~Nichani, R.~Ge, and J.~D. Lee}, {\em Smoothing the landscape
  boosts the signal for sgd: Optimal sample complexity for learning single
  index models}, arXiv preprint arXiv:2305.10633,  (2023).

\bibitem{diakonikolas20d}
{\sc I.~Diakonikolas, D.~M. Kane, V.~Kontonis, and N.~Zarifis}, {\em Algorithms
  and sq lower bounds for pac learning one-hidden-layer relu networks}, in
  Proceedings of Thirty Third Conference on Learning Theory, J.~Abernethy and
  S.~Agarwal, eds., vol.~125 of Proceedings of Machine Learning Research, PMLR,
  09--12 Jul 2020, pp.~1514--1539,
  \url{https://proceedings.mlr.press/v125/diakonikolas20d.html}.

\bibitem{dudeja2018learning}
{\sc R.~Dudeja and D.~Hsu}, {\em Learning single-index models in gaussian
  space}, in Conference On Learning Theory, PMLR, 2018, pp.~1887--1930.

\bibitem{Spokoiny01structure}
{\sc M.~Hristache, A.~Juditsky, J.~Polzehl, and V.~Spokoiny}, {\em Structure
  adaptive approach for dimension reduction}, The Annals of Statistics, 29
  (2001), pp.~1537--1566, \url{http://www.jstor.org/stable/2699943} (accessed
  2023-04-14).

\bibitem{hristache2001direct}
{\sc M.~Hristache, A.~Juditsky, and V.~Spokoiny}, {\em Direct estimation of the
  index coefficient in a single-index model}, Annals of Statistics,  (2001),
  pp.~595--623.

\bibitem{hsu2016heavytail}
{\sc D.~Hsu and S.~Sabato}, {\em Loss minimization and parameter estimation
  with heavy tails}, Journal of Machine Learning Research, 17 (2016),
  pp.~1--40, \url{http://jmlr.org/papers/v17/14-273.html}.

\bibitem{kearns98csq}
{\sc M.~Kearns}, {\em Efficient noise-tolerant learning from statistical
  queries}, J. ACM, 45 (1998), p.~983–1006,
  \url{https://doi.org/10.1145/293347.293351},
  \url{https://doi.org/10.1145/293347.293351}.

\bibitem{klock21estimating}
{\sc T.~Klock, A.~Lanteri, and S.~Vigogna}, {\em {Estimating multi-index models
  with response-conditional least squares}}, Electronic Journal of Statistics,
  15 (2021), pp.~589 -- 629, \url{https://doi.org/10.1214/20-EJS1785},
  \url{https://doi.org/10.1214/20-EJS1785}.

\bibitem{Laurent00adaptive}
{\sc B.~Laurent and P.~Massart}, {\em {Adaptive estimation of a quadratic
  functional by model selection}}, The Annals of Statistics, 28 (2000),
  pp.~1302 -- 1338, \url{https://doi.org/10.1214/aos/1015957395},
  \url{https://doi.org/10.1214/aos/1015957395}.

\bibitem{li91sir}
{\sc K.-C. Li}, {\em Sliced inverse regression for dimension reduction},
  Journal of the American Statistical Association, 86 (1991), pp.~316--327,
  \url{http://www.jstor.org/stable/2290563}.

\bibitem{Li92phd}
{\sc K.-C. Li}, {\em On principal hessian directions for data visualization and
  dimension reduction: Another application of stein's lemma}, Journal of the
  American Statistical Association, 87 (1992), pp.~1025--1039,
  \url{http://www.jstor.org/stable/2290640} (accessed 2023-04-14).

\bibitem{mousavi-hosseini2023neural}
{\sc A.~Mousavi-Hosseini, S.~Park, M.~Girotti, I.~Mitliagkas, and M.~A.
  Erdogdu}, {\em Neural networks efficiently learn low-dimensional
  representations with {SGD}}, in The Eleventh International Conference on
  Learning Representations, 2023,
  \url{https://openreview.net/forum?id=6taykzqcPD}.

\bibitem{nemirovsky83problem}
{\sc A.~S. Nemirovsky and D.~B. Yudin}, {\em Problem Complexity and Method
  Efficiency in Optimization}, Wiley-Interscience, 1983.

\bibitem{radhakrishnan2022feature}
{\sc A.~Radhakrishnan, D.~Beaglehole, P.~Pandit, and M.~Belkin}, {\em Feature
  learning in neural networks and kernel machines that recursively learn
  features}, arXiv preprint arXiv:2212.13881,  (2022).

\bibitem{Samarov93functional}
{\sc A.~M. Samarov}, {\em Exploring regression structure using nonparametric
  functional estimation}, Journal of the American Statistical Association, 88
  (1993), pp.~836--847, \url{http://www.jstor.org/stable/2290772} (accessed
  2023-12-07).

\bibitem{stewart90perturb}
{\sc G.~Stewart and J.-g. Sun}, {\em Matrix Perturbation Theory}, Academic
  Press, 1990.

\bibitem{Trivedi14EGOP}
{\sc S.~Trivedi, J.~Wang, S.~Kpotufe, and G.~Shakhnarovich}, {\em A consistent
  estimator of the expected gradient outerproduct}, in Proceedings of the
  Thirtieth Conference on Uncertainty in Artificial Intelligence, UAI'14,
  Arlington, Virginia, USA, 2014, AUAI Press, p.~819–828.

\bibitem{Xia02mave}
{\sc Y.~Xia, H.~Tong, W.~K. Li, and L.-X. Zhu}, {\em An adaptive estimation of
  dimension reduction space}, Journal of the Royal Statistical Society: Series
  B (Statistical Methodology), 64 (2002), pp.~363--410,
  \url{https://doi.org/https://doi.org/10.1111/1467-9868.03411},
  \url{https://rss.onlinelibrary.wiley.com/doi/abs/10.1111/1467-9868.03411},
  \url{https://arxiv.org/abs/https://rss.onlinelibrary.wiley.com/doi/pdf/10.1111/1467-9868.03411}.

\bibitem{Yang17stein}
{\sc Z.~Yang, K.~Balasubramanian, Z.~Wang, and H.~Liu}, {\em Estimating
  high-dimensional non-gaussian multiple index models via stein’s lemma}, in
  Advances in Neural Information Processing Systems, I.~Guyon, U.~V. Luxburg,
  S.~Bengio, H.~Wallach, R.~Fergus, S.~Vishwanathan, and R.~Garnett, eds.,
  vol.~30, Curran Associates, Inc., 2017,
  \url{https://proceedings.neurips.cc/paper_files/paper/2017/file/4db0f8b0fc895da263fd77fc8aecabe4-Paper.pdf}.

\end{thebibliography}

\newpage
\appendix

\section{SAVE Is Not Exhaustive}
\label{app:save}    
    In this section, we provide an example where the Sliced Average Variance Estimator (SAVE) proposed by \cite{cook00save} cannot extract all the relevant directions, and hence cannot recover the whole central mean subspace (CMS), let alone the central subspace (CS). In the setting of the multi-index model \cref{eqn:model}, SAVE extracts the index space by estimating the SAVE matrix: $$\vV = \E \left[\left(\E \left[ \vX \vX^T \mid \ Y \right] - \vI\right)^2 \right].$$ One critical condition needed for exhaustiveness here is that $\col(\vV)$ should contain all the relevant directions in the index space. This condition is not always satisfied. 
    
    Consider a simple noiseless single-index model: 
    \begin{align} \label{eqn:save_model}
        Y = f(\vu^{\T} \vX)\,,
    \end{align}
    where $\vX = (X_1, X_2)^{\T} \sim \Normal(\vzero, \vI_2)$, and $\vu = (1,0)^{\T}$. Next, we show that there is a class of link functions with which $\vV$ is a zero matrix.  Let $\xi(z_1, z_2) := \E_{Z \sim \Normal(0,1)}[Z^2 | Z \in [z_1, z_2]]$ for $0 \le z_1 \le 1 \le z_2$. Note that $\xi$ is continuous, and monotonically increasing in each of $z_1$ and $z_2$, $\xi(1,1) = \E[Z^2|Z=1] = 1$ and $\xi(0, \infty) = \E[Z^2|Z \ge 0] = 1$. 
    
    Therefore, there exists a continuous decreasing function 

    $\nu: [1, \infty) \to (0,1]$, such that $\nu(1) = 1, \lim_{z \to \infty}\nu(z) = 0$ and 
    \begin{align*}
        \xi (\nu(z),z) = 1 \text{ for any } z \in [1, \infty)\,.
    \end{align*} 

    Let $f_0 \colon [0,1] \to \R_+$ be any continuous, non-negative, and monotone increasing function defined on $[0,1]$ with $f_0 = 0$. Then, we define $f \colon \R \to \R_+ $ in terms of $f_0$ as follows:
    \begin{align}
    \label{eqn:save_link}
        f(z):=\left\{
            \begin{array}{lll}
               0,  &  & z < 0; \\
               f_0(z), &  & 0 \le z \le 1; \\
               f_0(\nu(z)), &  & 1 < z < \infty.
            \end{array}
        \right.\,
    \end{align}

    The following proposition shows that the SAVE matrix $\vV$ is a zero matrix with link function $f$.

       \begin{proposition}
        Suppose $(\vX, Y)$ is an example generated from Model \eqref{eqn:save_model} with link function as defined in \eqref{eqn:save_link}. Then, we have $\vV = 0$.
        \begin{proof}
            It suffices to show that for any $y \ge 0$, $\E[\vX \vX^{\T} \id(Y \ge y)] = \vI_2 \p(Y \ge y)$.
            For $y > f_0(1)$ or $y \le 0$, this is vacuously true.
            For $0 < y \le f_0(1)$, we can show from the definition of $f$ that: 
            \begin{align}
                \label{eqn:xy}
                \{Y \geq y\} = \{X_1: y \leq f(X_1) \leq f_0(1)\} = \{ f_0^{-1}(y) \leq X_1 \leq \nu^{-1}(f_0^{-1}(y))\}.
            \end{align}
            See also Figure~\ref{fig:link} for a graphical illustration. Therefore,
            \begin{align*}
                \E\left[\vX \vX^{\T}\Big|Y \ge y\right] 
                &= 
                \E\left[
                \begin{bmatrix}
                    X_1^2 & X_1 X_2 \\
                    X_1X_2 & X_2^2
                \end{bmatrix} \Bigg | Y \ge y \right] \\
                & = 
                \begin{bmatrix}
                    \E\left[X_1^2\Big|Y \ge y\right] & \E\left[X_1\Big|Y \ge y\right] \E[X_2] \\
                    \E\left[X_1\Big|Y \ge y\right] \E[X_2] & \E[X_2^2]
                \end{bmatrix} \\
                \alignedoverset{\eqref{eqn:xy}}{=}
                \begin{bmatrix}
                   \xi(f_0^{-1}(y), \nu^{-1}(f_0^{-1}(y))) & 0 \\
                    0 & 1
                \end{bmatrix} = \vI\,,
            \end{align*}

            Finally, we have:
            \begin{align*}
                \E\left[\vX \vX^{\T} \id(Y \ge y)\right] &= \E\left[\vX \vX^{\T} \mid Y \ge y \right] \p(Y \ge y) = \vI \p(Y \ge y)\,.
            \end{align*}
            The proof is completed.
        \end{proof}
    \end{proposition}
\begin{figure}
    \centering
    \includegraphics[scale=1.3]{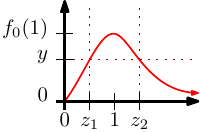}
    \caption{An example plot of the link function $f$ as defined in \eqref{eqn:save_link}. Here, we have $\{Y > y\} = [z_1, z_2]$ from the plot  $\{x: y \le f(x) \}$, where $z_1 = f_0^{-1}(y)$ and $z_2 = \nu^{-1}(z_1) = \nu^{-1}(f_0^{-1}(y))$.}
    \label{fig:link}
\end{figure}

\section{Connection to Local Linear Regression}
\label{sec:lle}
In this section, we will show that our estimator $\hat{\vbeta}_h(\shift):= \hat{\vbeta}_h(\shift; \Data)$, which is defined similarly as in \cref{eq:t_LR_rw_me} but with a full sample $\Data = \{(\vX_i, Y_i)\}_{i=1}^n$, is asymptotically equivalent to the following local linear estimator (LLE) of $\nabla g(\shift)$ with a data-dependent kernel: 
$$\hat{\vbeta}^{\LLE}_h(\shift) := \argmin_{(\alpha,\vbeta) \in \R \times \R^d} \sum_{i=1}^n \left(Y_i - \alpha - \vbeta^{\T} (\vX_i - \shift)\right)^2 K^{\shift}_h(\vX_i)\,,$$
where $K^{\shift}_h(\vx):= \rho_h(\vx; \shift)$.

\begin{proposition} \label{prop:LLE}
For each $\shift \in \R^d$,
    $$\left\lVert \hat{\vbeta}_h(\shift) - \hat{\vbeta}^{\LLE}_h(\shift) \right\rVert \overset{\p}{\rightarrow}{0}\,,$$
    here, the notation $\overset{\p}{\rightarrow}$ represents convergence in probability. 
    \begin{proof}

    By standard solution for the weighted least square error problems gives:
    \begin{align}
    \label{eqn:lle}
        \hat{\vbeta}^{\LLE}_h(\shift) = \vI_d^{\backslash 1}  (\vX_{\shift}^{\T} \vW^{\shift}_h \vX_{\shift})^{-1} \vX_{\shift}^{\T} \vW^{\shift}_h \vY
    \end{align}
    where $\vI_{d+1}^{\backslash 1} \in \R^{d \times (d+1)}$ is obtained by removing the first row of a $(d+1)\times (d+1)$ identity matrix $\vI_{d+1}$, $$\vX_{\shift} = \begin{bmatrix}
        1 & (\vX_1 - \shift)^{\T}\\
        \vdots & \vdots \\
        1 & (\vX_n - \shift)^{\T}
    \end{bmatrix}, \quad \vW_{\shift} = \diag(K^{\shift}_h(\vX_1), \ldots, K^{\shift}_h(\vX_n)), \quad \vY = \begin{bmatrix}
        Y_1\\
        \vdots \\
        Y_n
    \end{bmatrix}.$$
        One can show from \cref{eqn:lle} and the block matrix inversion formula (\Cref{lemma:matrix}) 
        that $$ \hat{\vbeta}_h(\shift) = h^{-2}\left( \hat{\Sigma}(\shift) \, \hat{\vbeta}^{\LLE}_h(\shift) + \hat{\mu}_{y}(\shift) \hat{\vmu}_{\vx}(\shift)\right), $$
    where
    \begin{align*}
        \hat{\Sigma}(\shift) &:= n^{-1}\sum_{i=1}^n \rho_h(\vX_i; \shift) (\vX_i - \shift) (\vX_i - \shift)^{\T} ;  \\
        \hat{\vmu}_{\vx}(\shift) &:= n^{-1}\sum_{i=1}^n \rho_h(\vX_i; \shift) (\vX_i - \shift)  ; \\
        \hat{\mu}_{y}(\shift) &:= n^{-1}\sum_{i=1}^n \rho_h(\vX_i; \shift) Y_i. 
    \end{align*}
    By the weak law of large numbers,
    $\hat{\Sigma}\overset{\p}{\rightarrow} h^2 \vI_d,$
    $\hat{\vmu}_{\vx}\overset{\p}{\rightarrow} 0,$
    $\hat{\mu}_{y} \overset{\p}{\rightarrow} \E_{\tilde{\vX} \sim \Normal(\vzero_d, \vI_d)} [f(\vU^{\T} \tilde{\vX})].$ The proposition follows and the proof is completed.
    \end{proof}
\end{proposition}

\section{Estimation Bias Under $P_{\vX}$ with a Bounded Support}
\label{sec:bounded_support}
In this section, we discuss the performance of our estimator when \Cref{assume:basic}(2) is not satisfied and the Lebesgue density $p$ of $P_{\vX}$ has a bounded support $\{\vx \in \R^d: \lVert \vx \rVert \le K\}$ for some $K < \infty$. In this case, we note that the smoothed gradient estimator $\hat{\vbeta}_h(\shift_j; \Data_{j,l})$ as defined in \cref{eq:t_LR_rw_me} is biased for ${\vbeta}_h(\shift_j)$. In particular, for a fixed $\shift_j$,
\begin{align*}
    \E[\hat{\vbeta}_h(\shift_j; \Data_{j,l})] &= {h^{-2}} \E[\rho_h(\vX; \shift) \cdot Y \cdot (\vX-\shift_j)]\\
    &= {h^{-2}} \int_{\lVert\vX\rVert \le K}  \rho_h(\vX; \shift_j) \cdot g(\vX) \cdot (\vX - \shift_j) d P_{\vX} \\
    &= {h^{-2}} \E_{\vZ\sim\Normal(\shift_j, h^2\vI_d)} \left[ \mathbbm{1}(\lVert\vZ\rVert \le K)  \cdot g(\vZ) \cdot (\vZ - \shift_j)\right]\,.
\end{align*}
Hence, the bias is given as:
\begin{align}
\label{eqn:bias}
    B_{h,j} &:= \E[\hat{\vbeta}_h(\shift_j; \Data_{j,l})] - {\vbeta}_h(\shift_j) \nonumber \\ 
    &= -{h^{-2}} \E_{\vZ\sim\Normal(\shift_j, h^2\vI_d)} \left[ \mathbbm{1}(\lVert\vZ\rVert > K)  \cdot g(\vZ) \cdot (\vZ - \shift_j)\right] 
\end{align}
Intuitively, $\lVert B_{h,j} \rVert$ should shrink to zero as $h \rightarrow 0$ when $K  \gg \lVert \shift_j\rVert$, which holds with high probability if we choose the variance $\sigma_{\theta}^2$ of $\shift_j$'s is small compared to $K$.
 
On the other hand, it is clear that if $\sup_{j \in [m]} \lVert B_{h,j} \rVert = o(n^{-1/2})$, then the bias will not affect the rate in the upper-bound for subspace estimation error. The following proposition shows how small $h$ needs to be to guarantee such small biases.

\begin{proposition}
    \label{prop:bias}
    Suppose that $\lVert \shift_j \rVert \le h$. Let $h = K (\sqrt{d} + \sqrt{\log n})^{-1}/4$, then we have the bias $\lVert \E \hat{\vbeta}_{j, \ell} - {\vbeta}_h(\shift_j) \rVert = O(\log n \cdot n^{-2/3})$.
\end{proposition} 

\begin{proof}
    By Cauchy-Schwarz Inequality and \cref{eqn:bias}, we have:
    \begin{align*}
        \lVert B_{h,j} \rVert  &\le {h^{-2}} \Bigg(\E_{\widetilde{\vZ}\sim\Normal(h^{-2}\shift_j, \vI_d)} \left[ \mathbbm{1}(\lVert\widetilde{\vZ}\rVert > Kh^{-1}) \right] \cdot \E_{\vZ\sim\Normal(\shift_j, h^2\vI_d)} \left[g(\vZ)^3\right] \nonumber \\ 
    & \qquad \qquad \qquad \cdot \E_{\vZ\sim\Normal(\shift_j, h^2\vI_d)}\left[\lVert\vZ - \shift_j\rVert^3\right] \Bigg)^{1/3}\,.
    \end{align*}
    We will upper-bound each of the terms in the parentheses.
    By \Cref{assume:design2} and \ref{assume:link}, it can be verified that there exists some $C > 0$, independent of $h$ and $n$, such that: 
    \begin{align}
        \label{eqn:bias_4}
        \E_{\vZ\sim\Normal(\shift_j, h^2\vI_d)} \left[g(\vZ)^3\right] \le C \cdot \ratiomoment \cdot \linkmoment\,.
    \end{align}
    The condition $\lVert\shift_j\rVert \le h$ implies that:
    \begin{align}
        \label{eqn:bias_5}
        \E_{\vZ\sim\Normal(\shift_j, h^2\vI_d)}\left[\lVert\vZ - \shift_j\rVert^3\right] \le C  \cdot h^3\,.
    \end{align}
    Finally, we have under the assumption $\lVert\shift_j\rVert < h < K/2$ that: 
    \begin{align}
    \label{eqn:bias_2}      \E_{\widetilde{\vZ}\sim\Normal(h^{-2}\shift_j, \vI_d)} \left[ \mathbbm{1}(\lVert\widetilde{\vZ}\rVert > Kh^{-1}) \right] & \le \p\left(\lVert \Normal(\vzero_d,\vI_d)\rVert > \frac{Kh^{-1}}{2}\right) =  \p\left(\chi^2_d > \frac{K^2h^{-2}}{4}\right) .
    \end{align}
    Consider the following concentration inequality for the chi-squared distribution:
    \begin{lemma}[Concentration for $\chi^2_d$, adapted from Lemma 1 of \cite{Laurent00adaptive}]
    \label{lemma:chi_square}
        Let $W$ be a random variable follows the centered chi-squared distribution with degree of freedom $d$, i.e., $W \sim \chi^2_d$, then for any $t>0$:
        $$\p(W > d+2\sqrt{dt}+2t) \le \exp(-t)\,.$$
    \end{lemma}
    By \Cref{lemma:chi_square}, when $h = K(\sqrt{d}+\sqrt{\log n})^{-1}/4$, we have:
    \begin{align}
    \label{eqn:bias_3}
       \p\left(\chi^2_d > \frac{K^2h^{-2}}{4}\right) \leq n^{-2}\,.
    \end{align}
    From \cref{eqn:bias_4} to \cref{eqn:bias_3}, we have:
    \begin{align*}
        \lVert B_{h,j} \rVert \le C h^{-2} (n^{-2} \cdot \ratiomoment \cdot \linkmoment \cdot h^3)^{1/3} \le \tilde{C} \cdot \log n \cdot  n^{-2/3}\,.
    \end{align*}
    where $\tilde{C}$ is independent of $n$. 
\end{proof}

\begin{remark}
    The condition $\lVert \shift_j \rVert \le h$ in \Cref{prop:bias} holds uniformly for $j \in [m]$, with probability at least $1-\delta$ with the choice $\sigma_{\theta} \lesssim h \cdot \log(m/\delta)$.
    This is consistent with our choice of $\sigma_{\theta}$ in the main theorem. 
\end{remark}

{
\section{The Case with Unknown Density Ratio} \label{sec:unknown_density_ratio}
 In this section, we consider the estimation of the smoothed gradient when the density ratio is unknown. Without loss of generality, we assume that the smoothing radius $h=1$ throughout this section and omit all $h$ in the subscript and simply write $\vbeta_j := \vbeta_h(\shift_j)$. Recall that for $j \in [m]$, $\ell = 1, 2$, we have the following smoothed gradient estimators in \Cref{sec:estimation}: $$\hat{\vbeta}_{j,l} = \frac{2m}{n} \sum_{(\vX, Y) \in \Data_{j, \ell}} \rho(\vX; \shift_j) \cdot Y \cdot (\vX-\shift_j), \quad $$
One natural solution to resolve the issue of unknown $\rho$ is to replace it by some estimates $\hat{\rho}_N$ based on another sample $\widetilde\vX_1, \ldots, \widetilde\vX_N$, and we get the following plug-in estimator:
$$\tilde{\vbeta}_{j,\ell} = \frac{2m}{n} \sum_{(\vX, Y) \in \Data_{j, \ell}} \hat{\rho}_N(\vX; \shift_j) \cdot Y \cdot (\vX-\shift_j). $$

The following proposition quantifies the effect of ``plugging in" on the $L_2$ error of estimating the smoothed gradients. 

\begin{proposition}
    Let $j \in [m]$, $\ell = 1,2$, and $\vx, \shift_j \in \R^d$.  Suppose that $\hat\rho_N(\vx; \shift_j)$ is an estimate of the density ratio $\rho(\vx; \shift_j)$ from a sample $\widetilde\vX_1, \ldots, \widetilde\vX_N$ with:
    \begin{align*}
       &\E\left[\left|\rho(\vX;\shift_j)\right|^{12}\right] < \infty \,,
       \E\left[\lVert\vX -\shift_j\rVert^{12}\right] < \infty \,,\\
        &a_{N,j} := \E\left[\left(\frac{\hat\rho_N(\vX; \shift_j)}{\rho(\vX; \shift_j)}\right)^{4}\right] \overset{N \to \infty}\longrightarrow 1, \\ 
        &b_{N,j} := \E\left[\left({\hat\rho_N(\vX; \shift_j)-\rho(\vX; \shift_j)}\right)^{2}\right] \overset{N \to \infty}\longrightarrow 0,
    \end{align*}
    where the expectation is taken over the randomness of $\vX, \widetilde{\vX}_1, \ldots, \widetilde{\vX}_N$. Then, we have for some $C>0$ independent of $n,m$ and $N$ such that:

    \begin{align*}
        \E\left[\lVert\tilde{\vbeta}_{j,\ell} - \vbeta_j \rVert^2 \right] \le C \cdot \left(\frac{m}{n} \cdot a_{N,j}^{n/4m} + b_{N,j}\right).
    \end{align*}

\end{proposition}
\begin{remark}
The performance of the plug-in estimator $\tilde\vbeta_{j,\ell}$ depends on both the relative error $a_{N,j}$ and the absolute error $b_{N,j}$ of the density ratio estimator $\hat\rho_N(\vx; \shift_j)$. The same squared $L_2$ error rate of $O(m/n)$ as the oracle estimator $\hat\vbeta_{j,\ell}$ can be obtained when $a_{N,j}^{n/4m}$ is bounded from above by a constant and $b_{N,j}$ is of smaller order than $O(m/n)$. 
\end{remark}

\begin{proof}
    We first apply the standard $L_2$ error decomposition: 
\begin{align}
    \label{eqn:decompose_error}
    \E[\lVert \tilde{\vbeta}_{j, \ell} - {\vbeta}_{j} \rVert^2] 
    &= \E[\lVert \tilde{\vbeta}_{j, \ell} - \E\tilde{\vbeta}_{j, \ell} \rVert^2]+ \lVert\E{\vbeta}_{j, \ell}  - {\vbeta}_{j}  \rVert^2 \nonumber\\ 
    &\le  \E[\lVert \tilde{\vbeta}_{j, \ell}\rVert^2]+ \lVert\E{\vbeta}_{j, \ell}  - {\vbeta}_{j}  \rVert^2.
\end{align}
In what follows, we provides upper bound the two terms on the right-hand side.
\begin{align*}
    \E [\lVert \tilde{\vbeta}_{j,\ell} \rVert^2] &= \E \left[\lVert \hat{\vbeta}_{j,\ell} \rVert^2 \cdot \prod_{\vX: (\vX, Y) \in \Data_{j, \ell}} \left(\frac{\hat{\rho}_N(\vX; \shift_j)}{{\rho}(\vX; \shift_j)}\right)^{2}\right] \\
    &\le \left(\E[\lVert\hat{\vbeta}_{j,\ell} \rVert^4]\right)^{1/2}
    \cdot \left(\E\left[\left(\frac{\hat{\rho}_N(\vX; \shift_j)}{{\rho}(\vX; \shift_j)}\right)^{4}\right]\right)^{n/4m}\\
    &= \left(\E\left[\left\lVert\frac{2m}{n} \sum_{(\vX, Y) \in \Data_{j, \ell}} \rho(\vX; \shift_j) \cdot Y \cdot (\vX-\shift_j) \right\rVert^4 \right]\right)^{1/2} \cdot a_N^{n/4m}\\
    &\le \frac{2m}{n} \cdot  \left(\E\left[\left\lVert \rho(\vX; \shift_j) \cdot Y \cdot (\vX-\shift_j) \right\rVert^4 \right]\right)^{1/2} \cdot a_N^{n/4m}\\
    & \le \frac{2m}{n} \cdot  \left(\E\left[\left|\rho(\vX;\shift_j)\right|^{12}\right]\E\left[\lVert\vX -\shift_j\rVert^{12}\right]\E\left[Y^{12}\right]\right)^{1/6} \cdot a_N^{n/4m} \\
    & \lesssim \frac{2m}{n} \cdot a_{N,j}^{n/4m}\,.
\end{align*}

\begin{align*}
    \lVert \E \tilde \vbeta_{j, \ell}  - {\vbeta}_{j}  \rVert &= \lVert \E[\left({\hat\rho_N(\vX; \shift_j)-\rho(\vX; \shift_j)}\right) \cdot (\vX - \shift_j) \cdot Y] \rVert \\
    &\le \left(\E[\left({\hat\rho_N(\vX; \shift_j)-\rho(\vX; \shift_j)}\right)^2]\right)^{1/2} \cdot \left(\E[Y^4]\right)^{1/4} \cdot \left(\E[\lVert\vX - \shift_j\rVert^4]\right)^{1/4}\\
    &\lesssim b_{N,j}^{1/2}\,.
\end{align*}
 The proof is completed by substituting these bounds in \Cref{eqn:decompose_error}.
\end{proof}

\section{Improving Dependence on $\delta$} \label{sec:mom}
Recall that in the upper-bound of our main theorem (\Cref{thm:generalrate}), the dependence on failure probability $\delta$ is of order $\tilde{O}(\delta^{-1/2})$ which may appear large in light of usual high-probability results, e.g., sub-gaussian concentrations. This is mainly due to the heavy-tailedness of the outer product estimates $\estimator_j$'s. We note that such dependence on $\delta$ is further improvable by the classical idea of median-of-mean estimators in robust statistics \cite{nemirovsky83problem}. Here, we present \Cref{alg:main_mom}, a median-of-mean variant of our estimator, adapted from Algorithm 2 of \cite{hsu2016heavytail}. The main difference from Algorithm \Cref{alg:main} lies in the way of aggregating the ASGOP estimates $\{\widehat{\vM}_j\}_{j \in [m]}$. Instead of taking a simple average, the multivariate median $\estimator_{\star}$ (c.f. Line 7) is taken as the final estimate of the ASGOP $\ASGOP$. It can be readily shown that with proper choice of $m \sim \log(1/\delta)$, one can establish an exponential concentration of $\estimator_{\star}$ to $\ASGOP$. Here, we outline the key steps:

\begin{itemize}
    \item By \Cref{lem:Z-error}, $$P(\lVert \estimator_j - \ASGOP\rVert \ge C_d \cdot \sqrt{m/n}) \le 2/3\,,$$ for some $C_d > 0$, which may depends on dimension $d$ and various parameters defined as in \Cref{sec:results}. Hence, Assumption 1 and 2 in \cite{hsu2016heavytail} is satisfied with $\varepsilon = C_d \cdot \sqrt{m/n}$. 
    \item By Proposition 9 of \cite{hsu2016heavytail}, we have:
    $$P\left(\lVert \estimator_{\star} - \ASGOP\rVert \ge 3 C_d \cdot \sqrt{m/n}\right) \le e^{-m/8}\,.$$
    \item Choose $m \approx 8 \log(1/\delta)$ to establish the desired results. 
\end{itemize}

\begin{algorithm}[tbp]
\caption{Estimation of the CMS $\minspace$ (Median-of-Means Variant)}\label{alg:main_mom}
\small 
\begin{algorithmic}[1]
    \vspace{.5em}
    \State \textbf{Input:} {$h > 0, \sigma_\theta > 0, m \in \mathbb{Z}$, dataset $\Data \subset \R^d \times \R$} 
    \State Randomly sample $\{\shift_{j}\}_{j=1}^{m}$ from $\Normal(\mathbf{0}_d, \sigma_\theta^2 \mathbf{I}_d )$ 
    \State Split $\Data$ into  $\Data_{1}$, $\Data_{2}$, \ldots, and $\Data_{m}$ with equal sizes

    \State Split each $\Data_{j}$ into $\Data_{j, 1}$ and $\Data_{j, 2}$ with equal sizes
    
    \State $\forall j \in [m]$, let $\hat{\vbeta}_{j,1} \gets \hat{\vbeta}_h(\shift_{j}, \Data_{j, 1})$, and  $\hat{\vbeta}_{j,2} \gets \hat{\vbeta}_h(\shift_{j}, \Data_{j,2})$  \Comment See (3.4) in the main paper

    \State $\forall j \in [m]$, $\estimator_j \gets  \frac{1}{2} (\hat{\vbeta}_{j,1}\hat{\vbeta}^{\T}_{j,2}+\hat{\vbeta}_{j,2}\hat{\vbeta}^{\T}_{j,1})$; and ${\cal M} \gets \{\estimator_j: j \in [m]\}$
    \State $\forall j \in [m]$, $r_j := \min\{r \ge 0: |B_r(\estimator_j) \cap {\cal M}| > m/2\}$; $j_{\star} := \arg\min_{j \in [m]} r_j$; and $\estimator_{\star} \gets \estimator_{j_{\star}}$
    
    \State \textbf{Return:} {$\widehat{\minspace}_{\star} \in \R^{d \times k}$, the top-$k$ eigenvectors of $\estimator_{\star}$}

\end{algorithmic}
\end{algorithm}
}

\newpage



\section*{Supplementary material}

# SUPPLEMENTARY MATERIALS: Efficient Estimation of the Central Mean Subspace via Smoothed Gradient Outer Products

Gan Yuan[*], Mingyue Xu[†], Samory Kpotufe[†], and Daniel Hsu[‡]

**SM1. Omitted Proof from the Main paper.**

**SM1.1. Proof of Proposition 3.11.**

*Proof.* By definition,

$$\mu_\rho^5 := \mathop{\mathbb{E}}_{\boldsymbol{X}\sim\mathcal{N}(\boldsymbol{0}_d,\sigma^2\boldsymbol{I}_d)} \left[ \left( \frac{\frac{1}{(2\pi)^{d/2}h^d}\exp\left(-\frac{\boldsymbol{X}^\intercal\boldsymbol{X}}{2h^2}\right)}{\frac{1}{(2\pi)^{d/2}\sigma^d}\exp\left(-\frac{\boldsymbol{X}^\intercal\boldsymbol{X}}{2\sigma^2}\right)} \right)^5 \right]$$

$$= \mathop{\mathbb{E}}_{\boldsymbol{X}\sim\mathcal{N}(\boldsymbol{0}_d,\boldsymbol{I}_d)} \left[ (\sigma/h)^{5d} \exp\left( \frac{5h^2-5\sigma^2}{2h^2\sigma^2} \boldsymbol{X}^\intercal\boldsymbol{X} \right) \right]$$

$$\text{(SM1.1)} \qquad = \frac{\sigma^{4d}}{h^{5d}} \cdot \int_{\mathbb{R}^d} \frac{1}{(2\pi)^{d/2}} \exp\left( \frac{4h^2-5\sigma^2}{2h^2\sigma^2} \boldsymbol{x}^\intercal\boldsymbol{x} \right) d\boldsymbol{x}\,.$$

When $h \geq \sqrt{5}\sigma/2$, $(4h^2-5\sigma^2)/(2h^2\sigma^2) \geq 0$ and thus $\mu_\rho = \infty$. When $0 < h < \sqrt{5}/2$:

$$\mu_\rho = \left(\frac{\sigma^{4d}}{h^{5d}}\right)^{1/5} \cdot \left(\frac{5\sigma^2-4h^2}{h^2\sigma^2}\right)^{-d/10} = \left(\frac{5h^8}{\sigma^8} - \frac{4h^{10}}{\sigma^{10}}\right)^{-d/10} < \infty\,.$$

This completes the proof. ∎

**SM1.2. Proof of Lemma 4.1 (Concentration of $\widetilde{M}$).**

*Proof of Lemma 4.1 (Concentration of $\widetilde{M}$).* Let $R > 0$, $\widetilde{\boldsymbol{Z}}_j := \boldsymbol{\beta}(\boldsymbol{\theta}_j)\boldsymbol{\beta}(\boldsymbol{\theta}_j)^\intercal - \overline{\boldsymbol{M}}$, and $\widetilde{\boldsymbol{Z}}'_j := \widetilde{\boldsymbol{Z}}_j \mathbb{1}(\|\widetilde{\boldsymbol{Z}}_j\|_{\mathrm{op}} \leq R)$. Note that

$$\widetilde{\boldsymbol{M}} - \overline{\boldsymbol{M}} = \frac{1}{m} \sum_{j=1}^m \widetilde{\boldsymbol{Z}}_j\,,$$

and

$$\mathbb{E}\left[\|\widetilde{\boldsymbol{Z}}_j\|_{\mathrm{op}}^2\right] = \mathbb{E}\left[\|\boldsymbol{\beta}(\boldsymbol{\theta}_j)\boldsymbol{\beta}(\boldsymbol{\theta}_j)^\intercal - \overline{\boldsymbol{M}}\|_{\mathrm{op}}^2\right]$$

$$= \mathbb{E}\left[\|\boldsymbol{U}[\overline{\nabla}_h f(\boldsymbol{\theta})\overline{\nabla}_h f(\boldsymbol{\theta})^\intercal - \mathbb{E}[\overline{\nabla}_h f(\boldsymbol{\theta})\overline{\nabla}_h f(\boldsymbol{\theta})]\boldsymbol{U}^\intercal\|_{\mathrm{op}}^2\right]$$

$$= \mathbb{E}\left[\|\overline{\nabla}_h f(\boldsymbol{\theta})\overline{\nabla}_h f(\boldsymbol{\theta})^\intercal - \mathbb{E}[\overline{\nabla}_h f(\boldsymbol{\theta})\overline{\nabla}_h f(\boldsymbol{\theta})^\intercal]\|_{\mathrm{op}}^2\right]$$

$$\leq 2\mu_\nabla^4\,.$$

---

[*]Department of Statistics, Columbia University, New York, NY, 10027 (gan.yuan@columbia.edu, skk2175@columbia.edu).
[†]Data Science Institute, Columbia University, New York, NY, 10027 (mx2258@columbia.edu).
[‡]Department of Computer Science and Data Science Institute, Columbia University, New York, NY, 10027 (djhsu@cs.columbia.edu).

Next, we establish upper-bounds for three critical quantities to invoke Lemma SM2.6:

$$\mathbb{P}\left(\|\widetilde{\boldsymbol{Z}}_j\|_{\mathrm{op}} > R\right) \leq \frac{1}{R^2} \mathbb{E}\left[\|\widetilde{\boldsymbol{Z}}_j\|_{\mathrm{op}}^2\right] = \frac{2\mu_\nabla^4}{R^2};$$

$$\left\|\mathbb{E}\left[\widetilde{\boldsymbol{Z}}_j - \widetilde{\boldsymbol{Z}}_j'\right]\right\|_{\mathrm{op}} \leq \mathbb{E}\left[\|\widetilde{\boldsymbol{Z}}_j\|_{\mathrm{op}} \mathbb{1}(\|\widetilde{\boldsymbol{Z}}_j\|_{\mathrm{op}} > R)\right] = \int_R^\infty \mathbb{P}\left(\|\widetilde{\boldsymbol{Z}}_j\|_{\mathrm{op}} > t\right) dt$$

$$\leq \int_R^\infty \frac{1}{t^2} \cdot \mathbb{E}\left(\|\widetilde{\boldsymbol{Z}}_j\|_{\mathrm{op}}^2\right) dt \leq \frac{2\mu_\nabla^4}{R} =: \Delta;$$

$$\left\|\sum_{j=1}^m \mathbb{E}[(\widetilde{\boldsymbol{Z}}_j' - \mathbb{E}\,\widehat{\boldsymbol{Z}}_j')^2]\right\|_{\mathrm{op}} \leq m\,\mathbb{E}\left[\|\widetilde{\boldsymbol{Z}}_j'^2\|_{\mathrm{op}}\right] \leq m\,\mathbb{E}\left[\|\widetilde{\boldsymbol{Z}}_j\|_{\mathrm{op}}^2\right] \leq 2m\mu_\nabla^4 =: \sigma^2.$$

Let $\tilde{\delta} \in (0, 1/2)$, set the threshold levels $R := \sigma/\sqrt{\tilde{\delta}}$. Then,

(SM1.2) $$\Delta = R\tilde{\delta}/m < R\log(2d/\tilde{\delta})/m,$$

(SM1.3) $$m\,\mathbb{P}\left(\|\widehat{\boldsymbol{Z}}_j\|_{\mathrm{op}} \geq R\right) \leq \tilde{\delta}, \text{ and}$$

(SM1.4) $$\sigma^2/R^2 = \tilde{\delta} < \log(2d/\tilde{\delta})/3,$$

Then, we apply Lemma SM2.6 with and $t := R\log(2d/\tilde{\delta})/m + \Delta$:

$$\mathbb{P}\left\{\|\widetilde{\boldsymbol{M}} - \overline{\boldsymbol{M}}\|_{\mathrm{op}} \geq \frac{2R\log(2d/\tilde{\delta})}{m}\right\} \stackrel{\text{(SM1.2)}}{\leq} \mathbb{P}\left\{\|\widetilde{\boldsymbol{M}} - \overline{\boldsymbol{M}}\|_{\mathrm{op}} \geq \frac{R\log(2d/\tilde{\delta})}{m} + \Delta\right\}$$

$$\stackrel{\text{Lemma SM2.6, (SM1.3)}}{\leq} 2d\exp\left(-\frac{\log^2(2d/\tilde{\delta})}{\sigma^2/R^2 + 2\log(2d/\tilde{\delta})/3}\right) + \tilde{\delta}$$

$$\stackrel{\text{(SM1.4)}}{\leq} 2\tilde{\delta}.$$

Finally, by setting $\delta = 2\tilde{\delta} \in (0,1)$ and substituting $R = \sigma/\sqrt{\tilde{\delta}}$, we have with probability at least $1 - \delta$,

$$\|\widetilde{\boldsymbol{M}} - \overline{\boldsymbol{M}}\|_{\mathrm{op}} \leq \left(\sqrt{2}\mu_\nabla^2 \cdot \frac{1}{\sqrt{m}}\right) \cdot \frac{1}{\sqrt{\delta}} \cdot \log\left(\frac{4d}{\delta}\right).$$

This concludes the proof. ∎

**SM1.3. Proof of Lemma 4.2 (Concentration of $\widehat{M}$).** In this section, we use the simplified notation $\boldsymbol{\beta}_j \equiv \boldsymbol{\beta}_h(\boldsymbol{\theta}_j)$, and define:

$$\widehat{\boldsymbol{M}}_j := \frac{1}{2}\left(\hat{\boldsymbol{\beta}}_{j,1}\hat{\boldsymbol{\beta}}_{j,2}^\mathsf{T} + \hat{\boldsymbol{\beta}}_{j,2}\hat{\boldsymbol{\beta}}_{j,1}^\mathsf{T}\right), \quad \widehat{\boldsymbol{Z}}_j := \widehat{\boldsymbol{M}}_j - \widetilde{\boldsymbol{M}}_j,$$

Then, $\widehat{\boldsymbol{M}} - \widetilde{\boldsymbol{M}} = m^{-1}\sum_{j=1}^m \widehat{\boldsymbol{Z}}_j$.

*Proof of Lemma 4.2 (Concentration of $\widehat{M}$).* Let $R > 0$ be the truncation level to be determined later, and $\widehat{\boldsymbol{Z}}_j' := \widehat{\boldsymbol{Z}}_j \mathbb{1}(\|\widehat{\boldsymbol{Z}}_j\|_{\mathrm{op}} \leq R)$. By standard but cumbersome calculations (of which the details we present after this proof), we can show that there exists some absolute constant $c_1 > 0$, such that for any $R > 0$,

(i) (Lemma SM1.4)
$$\mathbb{P}\left\{\left\|\widehat{\boldsymbol{Z}}_j\right\|_{\mathrm{op}} \geq R\right\} \leq c_1\left(\tilde{C}_{h,\sigma_\theta,1}^2 \cdot \frac{d^2(\mu_f+\sigma_Y)^4}{R^2} \cdot \frac{m^2}{n^2} + \tilde{C}_{h,\sigma_\theta,2}^2 \cdot \frac{d(\mu_f+\sigma_Y)^2}{R^2} \cdot \frac{m}{n}\right);$$

(ii) (Lemma SM1.5)
$$\left\|\mathbb{E}\left[\widehat{\boldsymbol{Z}}_j - \widehat{\boldsymbol{Z}}_j'\right]\right\|_{\mathrm{op}} \leq \mathbb{E}\left[\left\|\widehat{\boldsymbol{Z}}_j\right\|_{\mathrm{op}} \mathbb{1}(\|\widehat{\boldsymbol{Z}}_j\|_{\mathrm{op}} > R)\right]$$
$$\leq c_1\left(\tilde{C}_{h,\sigma_\theta,1}^2 \cdot \frac{d^2(\mu_f+\sigma_Y)^4}{R} \cdot \frac{m^2}{n^2} + \tilde{C}_{h,\sigma_\theta,2}^2 \cdot \frac{d(\mu_f+\sigma_Y)^2}{R} \cdot \frac{m}{n}\right)$$
$$=: \Delta;$$

(iii)
$$\left\|\sum_{j=1}^m \mathbb{E}[(\widehat{\boldsymbol{Z}}_j' - \mathbb{E}\,\widehat{\boldsymbol{Z}}_j')^2]\right\|_{\mathrm{op}} \leq m\,\mathbb{E}\left[\left\|\widehat{\boldsymbol{Z}}_j^2\right\|_{\mathrm{op}}\right] \leq m\,\mathbb{E}\left[\left\|\widehat{\boldsymbol{Z}}_j\right\|_{\mathrm{op}}^2\right]$$
$$\overset{\text{Lemma SM1.2}}{\leq} c_1\left(\tilde{C}_{h,\sigma_\theta,1}^2 \cdot d^2(\mu_f+\sigma_Y)^4 \cdot \frac{m^3}{n^2} + \tilde{C}_{h,\sigma_\theta,2}^2 \cdot d(\mu_f+\sigma_Y)^2 \cdot \frac{m^2}{n}\right)\right] =: \sigma^2.$$

Let $\tilde{\delta} \in (0, 1/2)$, set the threshold levels $R := \sigma/\sqrt{\tilde{\delta}}$. Then,

(SM1.5) $$\Delta = R\tilde{\delta}/m < R\log(2d/\tilde{\delta})/m,$$

(SM1.6) $$m\,\mathbb{P}\left\{\left\|\widehat{\boldsymbol{Z}}_j\right\|_{\mathrm{op}} \geq R\right\} \leq \tilde{\delta},$$

and

(SM1.7) $$\sigma^2/R^2 = \tilde{\delta} < \log(2d/\tilde{\delta})/3,$$

Then, we apply Lemma SM2.6 with and $t := R\log(2d/\tilde{\delta})/m + \Delta$:

$$\mathbb{P}\left\{\|\widehat{\boldsymbol{M}} - \boldsymbol{M}\|_{\mathrm{op}} \geq \frac{2R\log(2d/\tilde{\delta})}{m}\right\} \overset{\text{(SM1.5)}}{\leq} \mathbb{P}\left\{\|\widehat{\boldsymbol{M}} - \boldsymbol{M}\|_{\mathrm{op}} \geq \frac{R\log(2d/\tilde{\delta})}{m} + \Delta\right\}$$
$$\overset{\text{Lemma SM2.6, (SM1.6)}}{\leq} 2d\exp\left(-\frac{\log^2(2d/\tilde{\delta})}{\sigma^2/R^2 + 2\log(2d/\tilde{\delta})/3}\right) + \tilde{\delta}$$
$$\overset{\text{(SM1.7)}}{\leq} 2\tilde{\delta}.$$

Finally, by setting $\delta = 2\tilde{\delta} \in (0,1)$ and substituting $R = \sigma/\sqrt{\tilde{\delta}}$, we have with probability at least $1 - \delta$,

$$\|\widehat{\boldsymbol{M}} - \widetilde{\boldsymbol{M}}\|_{\mathrm{op}} \leq C\left(\tilde{C}_{h,\sigma_\theta,1} \cdot d(\mu_f+\sigma_Y)^2 \cdot \frac{\sqrt{m}}{n} + \tilde{C}_{h,\sigma_\theta,2} \cdot d^{1/2}(\mu_f+\sigma_Y) \cdot \frac{1}{\sqrt{n}}\right)\frac{\log(4d/\delta)}{\sqrt{\delta}}.$$

This concludes the proof. ∎

The following Lemmas SM1.1-SM1.5 provides the detailed calculations to justify (i), (ii), and (iii) in the above proof of Lemma 4.2.

**Lemma SM1.1.** *Under Assumption 3 and 5, we have*

$$\mathbb{E}[Y^6] \leq \left(\mu_f + \sqrt{6}e^{1/e}\sigma_Y\right)^6.$$

*Proof of Lemma SM1.1.* Let $\epsilon = Y - \mathbb{E}[Y \mid \boldsymbol{X}] = Y - f(\boldsymbol{U}^\mathsf{T}\boldsymbol{X})$, then

$$\begin{aligned}
\mathbb{E}[Y^6] &= \mathbb{E}\left[(f(\boldsymbol{U}^\mathsf{T}\boldsymbol{X}) + \epsilon)^6\right] \\
&\stackrel{\text{Minkowski's Ineq.}}{\leq} \left(\left(\mathbb{E}[f^6(\boldsymbol{U}^\mathsf{T}\boldsymbol{X})]\right)^{1/6} + \left(\mathbb{E}[\epsilon^6]\right)^{1/6}\right)^6 \\
&\stackrel{\text{Fact SM2.4}}{\leq} \left(\left(\mathbb{E}[f^6(\boldsymbol{U}^\mathsf{T}\boldsymbol{X})]\right)^{1/6} + \sqrt{6}e^{1/e}\sigma_Y\right)^6 \\
&= \left(\mu_f + \sqrt{6}e^{1/e}\sigma_Y\right)^6.
\end{aligned}$$
∎

**Lemma SM1.2.** *Under Assumption 1-5, there exists some absolute constant $c_1 > 0$ such that,*

$$\mathop{\mathbb{E}}_{P_{\boldsymbol{X},Y}}\left[\|\widehat{\boldsymbol{Z}}_j\|_{\text{op}}^2\right] \leq c_1\left(\tilde{C}_{h,\sigma_\theta,1}^2 \cdot d^2(\mu_f + \sigma_Y)^4 \cdot \frac{m^2}{n^2} + \tilde{C}_{h,\sigma_\theta,2}^2 \cdot d(\mu_f + \sigma_Y)^2 \cdot \frac{m}{n}\right).$$

*Proof of Lemma SM1.2.* For simplicity, we write $\boldsymbol{\beta}_j = \boldsymbol{\beta}_h(\boldsymbol{\theta}_j)$. Note that for $\ell = 1, 2$,

$$\begin{aligned}
\mathbb{E}\left[\|\hat{\boldsymbol{\beta}}_{j,\ell} - \boldsymbol{\beta}_j\|^2 \mid \boldsymbol{\theta}_j\right] &\leq \frac{2mh^{-2}}{n}\mathbb{E}\left[\|\rho_h(\boldsymbol{X};\boldsymbol{\theta})(\boldsymbol{X} - \boldsymbol{\theta})Y\|^2 \mid \boldsymbol{\theta} = \boldsymbol{\theta}_j\right] \\
&\stackrel{\text{Hölder's Ineq.}}{\leq} \frac{2mh^{-2}}{n}\left\{\mathbb{E}\left[\|\rho_h(\boldsymbol{X};\boldsymbol{\theta})(\boldsymbol{X} - \boldsymbol{\theta})\|^3 \mid \boldsymbol{\theta} = \boldsymbol{\theta}_j\right]\right\}^{2/3}\left\{\mathbb{E}\left[Y^6\right]\right\}^{1/3}
\end{aligned}$$

By Lemma SM1.1, $\mathbb{E}[Y^6] \leq (\mu_f + \sqrt{6}e^{1/e}\sigma)^6 \leq c_2(\mu_f + \sigma)^6$. It suffices to upper-bound the term $\mathbb{E}[\|\rho_h(\boldsymbol{X};\boldsymbol{\theta})(\boldsymbol{X} - \boldsymbol{\theta})\|^3 \mid \boldsymbol{\theta} = \boldsymbol{\theta}_j]$. Let $\tilde{\rho}_h(\boldsymbol{x};\boldsymbol{\theta}) := \varphi_h(\boldsymbol{x} - \boldsymbol{\theta})/\varphi_h(\boldsymbol{x})$ be the density ratio

between $\mathcal{N}(\boldsymbol{\theta}, h^2\boldsymbol{I}_d)$ and $\mathcal{N}(\boldsymbol{0}_d, h^2\boldsymbol{I}_d)$, then $\rho_h(\boldsymbol{X}; \boldsymbol{\theta}) = (\varphi_h(\boldsymbol{X})/p(\boldsymbol{X})) \cdot \tilde{\rho}_h(\boldsymbol{X}; \boldsymbol{\theta})$. Hence,

$$\mathbb{E}_{\boldsymbol{X} \sim P_{\boldsymbol{X}}}\left[\|\rho_h(\boldsymbol{X}; \boldsymbol{\theta})(\boldsymbol{X} - \boldsymbol{\theta})\|^3 \mid \boldsymbol{\theta} = \boldsymbol{\theta}_j\right]$$

$$= \mathbb{E}\left[\left(\frac{\varphi_h(\boldsymbol{X})}{p(\boldsymbol{X})} \cdot \tilde{\rho}_h(\boldsymbol{X}; \boldsymbol{\theta})\right)^{5/2} \cdot (\rho_h(\boldsymbol{X}; \boldsymbol{\theta}))^{1/2} \cdot \|(\boldsymbol{X} - \boldsymbol{\theta})\|^3 \mid \boldsymbol{\theta} = \boldsymbol{\theta}_j\right]$$

$$\stackrel{\text{Cauchy-Schwarz Ineq.}}{\leq} \left\{\mathbb{E}\left[\left(\frac{\varphi_h(\boldsymbol{X})}{p(\boldsymbol{X})}\right)^5\right]\mathbb{E}\left[\left(\frac{\varphi_h(\boldsymbol{X} - \boldsymbol{\theta})}{p(\boldsymbol{X})}\right)\tilde{\rho}_h(\boldsymbol{X}; \boldsymbol{\theta})^5 \|\boldsymbol{X} - \boldsymbol{\theta}\|^6 \mid \boldsymbol{\theta} = \boldsymbol{\theta}_j\right]\right\}^{1/2}$$

$$\stackrel{\text{Assumption 4}}{\leq} \mu_\rho^{5/2}\left\{\int \varphi_h(\boldsymbol{z})\exp\left(\frac{5}{h^2}\boldsymbol{z}^\mathsf{T}\boldsymbol{\theta}_j + \frac{5}{2h^2}\|\boldsymbol{\theta}_j\|^2\right)\|\boldsymbol{z}\|^6 d\boldsymbol{z}\right\}^{1/2}$$

$$\leq \mu_\rho^{5/2}\left\{\exp\left(15\|\boldsymbol{\theta}_j\|^2\right)\mathbb{E}_{\boldsymbol{Z} \sim \mathcal{N}(\boldsymbol{0}, h^2\boldsymbol{I}_d)}\left[\|\boldsymbol{Z} + 5\boldsymbol{\theta}\|^6 \mid \boldsymbol{\theta} = \boldsymbol{\theta}_j\right]\right\}^{1/2}$$

$$\stackrel{\text{Fact SM2.3}}{\leq} c_2 \mu_\rho^{5/2} \exp\left(\frac{15}{2h^2}\|\boldsymbol{\theta}_j\|^2\right)(h^3 d^{3/2} + \|\boldsymbol{\theta}_j\|^3)$$

$$\leq c_3 \mu_\rho^{5/2} h^3 d^{3/2} \exp\left(\frac{15}{2h^2}\|\boldsymbol{\theta}_j\|^2\right),$$

for some absolute constant $c_2, c_3 > 0$, and thus

$$\mathbb{E}\left[\|\hat{\boldsymbol{\beta}}_{j,\ell} - \boldsymbol{\beta}_j\|^2 \mid \boldsymbol{\theta} = \boldsymbol{\theta}_j\right] \leq \frac{c_4 \mu_\rho^{5/3} d(\mu_f + \sigma_Y)^2 m}{n} \cdot \exp\left(\frac{5}{h^2}\|\boldsymbol{\theta}_j\|^2\right),$$

for some $c_4 > 0$ and

$$\mathbb{E}\left[\left\|\widehat{\boldsymbol{Z}}_j\right\|_{\text{op}}^2\right] = \mathbb{E}\left[\left\|\hat{\boldsymbol{\beta}}_{j,1}\hat{\boldsymbol{\beta}}_{j,2}^\mathsf{T} - \boldsymbol{\beta}_j\boldsymbol{\beta}_j^\mathsf{T}\right\|_{\text{op}}^2\right]$$

$$\leq 2\mathbb{E}\left[\mathbb{E}\left[\|\hat{\boldsymbol{\beta}}_{j,1} - \boldsymbol{\beta}_j\|^2 \mid \boldsymbol{\theta}_j\right]\mathbb{E}\left[\|\hat{\boldsymbol{\beta}}_{j,2} - \boldsymbol{\beta}_j\|^2 \mid \boldsymbol{\theta}_j\right] + 2\|\boldsymbol{\beta}_j\|^2\mathbb{E}\left[\|\hat{\boldsymbol{\beta}}_{j,1} - \boldsymbol{\beta}_j\|^2 \mid \boldsymbol{\theta}_j\right]\right]$$

$$\leq c_1\left(\frac{\mu_\rho^{10/3} d^2(\mu_f + \sigma_Y)^4 m^2}{n^2}\cdot\left(1 - \frac{20\sigma_\theta^2}{h^2}\right)^{-d/2} + \frac{\mu_\nabla^2 \mu_\rho^{5/3} d(\mu_f + \sigma_Y)^2 m}{n}\cdot\left(1 - \frac{10\sigma_\theta^2}{h^2}\right)^{-d/2}\right)$$

$$=: c_1\left(\tilde{C}_{h,\sigma_\theta,1}^2 \cdot d^2(\mu_f + \sigma_Y)^4 \cdot \frac{m^2}{n^2} + \tilde{C}_{h,\sigma_\theta,2}^2 \cdot d(\mu_f + \sigma_Y)^2 \cdot \frac{m}{n}\right),$$

for some $c_1 > 0$, where in the second last step we use the fact $\mathbb{E}\left[\|\boldsymbol{\beta}_j\|^2\right] \leq \sqrt{\mathbb{E}\left[\|\boldsymbol{\beta}_j\|^4\right]} \leq \mu_\nabla^2$, and $\mathbb{E}[t\|\boldsymbol{\theta}_j\|^2] = (1 - 2\sigma_\theta t)^{-d/2}, \forall t < 1/(2\sigma_\theta)$. This concludes the proof. ■

*Remark* SM1.3. The upper-bound in Lemma SM1.2 and Lemma 4.2 is related implicitly to the largest eigenvalue of the ESGOP $\lambda_1(\overline{\boldsymbol{M}})$. In particular, notice that we have the following two upper-bounds for $\lambda_1(\overline{\boldsymbol{M}})$: 1) $\lambda_1(\overline{\boldsymbol{M}}) \leq \mathbb{E}_{\boldsymbol{\theta} \sim \mathcal{N}(\boldsymbol{0}_d, \sigma_\theta^2 \boldsymbol{I}_d)}\left[\|\boldsymbol{\beta}(\boldsymbol{\theta})\boldsymbol{\beta}(\boldsymbol{\theta})^\mathsf{T}\|\right] \leq \mu_\nabla$; and 2) $\lambda_1(\overline{\boldsymbol{M}}) \leq \mathbb{E}_{\boldsymbol{\theta} \sim \mathcal{N}(\boldsymbol{0}_d, \sigma_\theta^2 \boldsymbol{I}_d)}\left[\|\boldsymbol{\beta}(\boldsymbol{\theta})\boldsymbol{\beta}(\boldsymbol{\theta})^\mathsf{T}\|\right] \leq c_4 \mu_\rho^{5/3} d(\mu_f + \sigma_Y)(1 - 5\sigma_{\theta^2}/h^2)^{-d/2}$.

**Lemma SM1.4.** *For any $j \in [m]$ and $R > 0$,*

$$\mathbb{P}\left\{\left\|\widehat{\boldsymbol{Z}}_j\right\|_{\mathrm{op}} \geq R\right\} \leq c_1 \left(\tilde{C}_{h,\sigma_\theta,1}^2 \cdot \frac{d^2(\mu_f + \sigma_Y)^4}{R^2} \cdot \frac{m^2}{n^2} + \tilde{C}_{h,\sigma_\theta,2}^2 \cdot \frac{d(\mu_f + \sigma_Y)^2}{R^2} \cdot \frac{m}{n}\right).$$

*Proof of Lemma SM1.4.* It follows from Lemma SM1.2 and Markov inequality. ∎

**Lemma SM1.5.** *For any $j \in [m]$ and $R > 0$,*

$$\mathbb{E}\left[\|\widehat{\boldsymbol{Z}}_j\|_{\mathrm{op}}\mathbb{1}(\|\widehat{\boldsymbol{Z}}_j\|_{\mathrm{op}} > R)\right] \leq c_1 \left(\tilde{C}_{h,\sigma_\theta,1}^2 \cdot \frac{d^2(\mu_f + \sigma_Y)^4}{R} \cdot \frac{m^2}{n^2} + \tilde{C}_{h,\sigma_\theta,2}^2 \cdot \frac{d(\mu_f + \sigma_Y)^2}{R} \cdot \frac{m}{n}\right).$$

*Proof of Lemma SM1.5.* It follows from Lemma SM1.4 and the fact that for any nonnegative random variable $\xi \geq 0$:

$$\mathbb{E}\left[\xi\mathbb{1}(\xi > R)\right] = \int_R^\infty \mathbb{P}\{\xi > t\}\, dt.$$

∎

**SM1.4. Proof of Lemma 4.3.** Throughout the section, we use the short notation $\mathcal{I}_{r,k} := \{\boldsymbol{\alpha} \in \mathbb{N}^k$ such that $\|\boldsymbol{\alpha}\|_1 \leq r\}$. The following lemma gives an upper-bound for $|\mathcal{I}_{r,k}|$, or equivalently the dimension of the space spanned by polynomials of degree at most $r$.

**Lemma SM1.6.** *For any integers $r \geq 0$ and $k \geq 1$,*

$$|\mathcal{I}_{r,k}| \leq (r+1)^k.$$

*Proof of Lemma SM1.6.* The upper-bound can be achieved by:

$$|\mathcal{I}_{r,k}| = \left\|\{\boldsymbol{\alpha} \in \mathbb{N}^k : \|\boldsymbol{\alpha}\|_1 \leq r\}\right\| = \sum_{l=0}^r \left\|\{\boldsymbol{\alpha} \in \mathbb{N}^k : \|\boldsymbol{\alpha}\|_1 = l\}\right\|$$

$$\leq \sum_{l=0}^r (l+1)^{(k-1)} \leq (r+1)^k.$$

∎

The smoothed gradient $\overline{\nabla}_h f(\boldsymbol{\theta})$ is defined as an expectation under Gaussian with variance $h^2 \boldsymbol{I}$, which deviates from the normalization assumption $\mathbb{E}[f(\boldsymbol{Z})^2] = 1$ in which the expectation is w.r.t. standard Gaussian. The next lemma shows at most how large these expectations can differ from each other for the polynomial class we are considering.

**Lemma SM1.7.** *Let $h > 0$, we have under Assumption 6 that*

$$(r+1)^{-2k-(r/2)} \cdot \min\{h^{2r}, 1\} \leq \mathop{\mathbb{E}}_{\boldsymbol{Z} \sim \mathcal{N}(\boldsymbol{0}, h^2 \boldsymbol{I}_d)}[f(\boldsymbol{Z})^2] \leq (r+1)^{2k+(r/2)} \cdot \max\{h^{2r}, 1\}.$$

*Proof of Lemma* SM1.7. We first show the upper-bound of $\mathbb{E}_{\boldsymbol{Z}\sim\mathcal{N}(\boldsymbol{0},h^2\boldsymbol{I}_d)}[f(\boldsymbol{Z})^2]$. Since $f$ admits a Hermite expansion $f(\boldsymbol{z}) = \sum_{\boldsymbol{\alpha}\in\mathcal{I}_{r,k}} w_{\boldsymbol{\alpha}} H_{\boldsymbol{\alpha}}(\boldsymbol{z})$, with $\sum_{\boldsymbol{\alpha}\in\mathcal{I}_{r,k}} w_{\boldsymbol{\alpha}}^2 = 1$. Note that

$$\mathbb{E}_{\boldsymbol{Z}\sim\mathcal{N}(\boldsymbol{0}_k,h^2\boldsymbol{I}_k)}[f(\boldsymbol{Z})^2] = \mathbb{E}_{\boldsymbol{Z}\sim\mathcal{N}(\boldsymbol{0}_k,\boldsymbol{I}_k)}[f(h\boldsymbol{Z})^2]$$

$$= \mathbb{E}_{\boldsymbol{Z}\sim\mathcal{N}(\boldsymbol{0}_k,\boldsymbol{I}_k)}\left[\left(\sum_{\boldsymbol{\alpha}\in\mathcal{I}_{r,k}} w_{\boldsymbol{\alpha}} H_{\boldsymbol{\alpha}}(h\boldsymbol{Z})\right)^2\right]$$

$$\leq \mathbb{E}_{\boldsymbol{Z}\sim\mathcal{N}(\boldsymbol{0}_k,\boldsymbol{I}_k)}\left[\left(\sum_{\boldsymbol{\alpha}\in\mathcal{I}_{r,k}} w_{\boldsymbol{\alpha}}^2\right)\left(\sum_{\boldsymbol{\alpha}\in\mathcal{I}_{r,k}} H_{\boldsymbol{\alpha}}(h\boldsymbol{Z})^2\right)\right]$$

$$\text{(SM1.8)} \quad \leq |\mathcal{I}_{r,k}| \cdot \max_{\boldsymbol{\alpha}\in\mathcal{I}_{r,k}} \mathbb{E}_{\boldsymbol{Z}\sim\mathcal{N}(\boldsymbol{0}_k,\boldsymbol{I}_k)}\left[H_{\boldsymbol{\alpha}}(h\boldsymbol{Z})^2\right]$$

Since $|\mathcal{I}_{r,k}| \leq (r+1)^k$ by Lemma SM1.6, it remains to find an upper-bound (uniformly in $\boldsymbol{\alpha}$) for $\mathbb{E}_{\boldsymbol{Z}\sim\mathcal{N}(\boldsymbol{0}_k,\boldsymbol{I}_k)}\left[H_{\boldsymbol{\alpha}}(h\boldsymbol{Z})^2\right]$:

$$\mathbb{E}_{\boldsymbol{Z}\sim\mathcal{N}(\boldsymbol{0}_k,\boldsymbol{I}_k)}\left[H_{\boldsymbol{\alpha}}(h\boldsymbol{Z})^2\right] := \mathbb{E}\left[\prod_{j=1}^{k} H_{\alpha_j}(hZ_j)^2\right]$$

$$\overset{\text{Fact SM2.1}(iii)}{=} \mathbb{E}\left[\prod_{j=1}^{k}\left(\sqrt{\alpha_j!}\sum_{i=0}^{\lfloor\alpha_j/2\rfloor}\frac{h^{\alpha_j-2i}(h^2-1)^i}{2^i i!\sqrt{(\alpha_j-2i)!}}H_{\alpha_j-2i}(Z_j)\right)^2\right]$$

$$= \left(\sum_{i=0}^{\lfloor\alpha_j/2\rfloor}\frac{\alpha_j!}{(i!)^2(\alpha_j-2i)!}\cdot h^{2\alpha_j}\cdot\prod_{j=1}^{k}\left(\frac{1-h^{-2}}{2}\right)^{2i}\right).$$

By Stirling's approximation that for integer $a$: $a! \sim (a/e)^a \cdot (2\pi a)^{1/2}$, one can show that for any $i \in [\alpha_j/2 + 1]$:

$$\frac{\alpha_j!}{(i!)^2(\alpha_j-2i)!} \leq 2\alpha_j^{\alpha_j/2} \leq 2r^{r/2}.$$

Therefore, we have:

$$\mathbb{E}_{\boldsymbol{Z}\sim\mathcal{N}(\boldsymbol{0}_k,\boldsymbol{I}_k)}\left[H_{\boldsymbol{\alpha}}(h\boldsymbol{Z})^2\right] \leq r^{r/2} \cdot \prod_{j=1}^{k} h^{2\alpha_j} \cdot \left(\sum_{i=0}^{\lfloor\alpha_j/2\rfloor}\left(\frac{1-h^{-2}}{2}\right)^{2i}\right).$$

When $h \geq \sqrt{3}/3$, we have $(1-h^{-2})/2 \in [-1, 1/2)$, thus

$$\sum_{i=0}^{\lfloor\alpha_j/2\rfloor}\left(\frac{1-h^{-2}}{2}\right)^{2i} \leq \frac{1}{1-((1-h^{-2})/2)^2} \leq 4/3,$$

and hence:

$$\mathbb{E}_{\boldsymbol{Z}\sim\mathcal{N}(\boldsymbol{0}_k,\boldsymbol{I}_k)}\left[H_{\boldsymbol{\alpha}}(h\boldsymbol{Z})^2\right] \leq h^{2\|\boldsymbol{\alpha}\|} \cdot r^{r/2} \cdot (4/3)^k.$$

When $h < \sqrt{3}/3$, using the inequality $(a^{q+2} - 1)/(a^2 - 1) \leq (q+2)a^q/2$ for any $a > 1$ and $q \in \mathbb{N}$, we have:

$$\sum_{i=0}^{\lfloor \alpha_j/2 \rfloor} \left(\frac{1-h^{-2}}{2}\right)^{2i} \leq \frac{((h^{-2}-1)/2)^{\alpha_j+2} - 1}{((h^{-2}-1)/2)^2 - 1} \leq \frac{\alpha_j + 2}{2} \cdot \left(\frac{h^{-2}-1}{2}\right)^{\alpha_j},$$

and hence:

$$\mathop{\mathbb{E}}_{\boldsymbol{Z} \sim \mathcal{N}(\boldsymbol{0}_k, \boldsymbol{I}_k)} \left[H_{\boldsymbol{\alpha}}(h\boldsymbol{Z})^2\right] \leq r^{r/2} \cdot \left(\frac{r+2}{2}\right)^k \cdot \left(\frac{1-h^2}{2}\right)^{\|\boldsymbol{\alpha}\|_1} \leq r^{r/2} \cdot \left(\frac{r+2}{2}\right)^k.$$

By combining the two cases, we have:

$$\mathop{\mathbb{E}}_{\boldsymbol{Z} \sim \mathcal{N}(\boldsymbol{0}_k, \boldsymbol{I}_k)} \left[H_{\boldsymbol{\alpha}}(h\boldsymbol{Z})^2\right] \leq r^{r/2} \cdot \left(\frac{r+2}{2}\right)^k \cdot \max\{h^{2r}, 1\} \leq (r+1)^{k+(r/2)} \cdot \max\{h^{2r}, 1\}.$$

and the upper-bound for $\mathbb{E}_{\boldsymbol{Z} \sim \mathcal{N}(\boldsymbol{0}_k, h^2 \boldsymbol{I}_k)}[f(\boldsymbol{Z})^2]$ follows.

To get the lower-bound, we let $f_h(\boldsymbol{z}) = f(h\boldsymbol{z})$, and hence $\mathbb{E}_{\boldsymbol{Z} \sim \mathcal{N}(\boldsymbol{0}_k, h^{-2} \boldsymbol{I}_k)}[f_h(\boldsymbol{Z})] = \mathbb{E}_{\boldsymbol{Z} \sim \mathcal{N}(\boldsymbol{0}_k, \boldsymbol{I}_k)}[f(\boldsymbol{Z})] = 1$. On the other hand, we have by applying the upper-bound with $h$ replaced by $h^{-1}$ that:

$$1 = \mathop{\mathbb{E}}_{\boldsymbol{Z} \sim \mathcal{N}(\boldsymbol{0}_k, h^{-2} \boldsymbol{I}_k)}[f_h(\boldsymbol{Z})^2] \leq (r+1)^{2k+(r/2)} \cdot \max\{h^{-2r}, 1\} \cdot \mathop{\mathbb{E}}_{\boldsymbol{Z} \sim \mathcal{N}(\boldsymbol{0}_k, \boldsymbol{I}_k)}[(f_h(\boldsymbol{Z}))^2]$$

$$= (r+1)^{2k+(r/2)} \cdot \max\{h^{-2r}, 1\} \cdot \mathop{\mathbb{E}}_{\boldsymbol{Z} \sim \mathcal{N}(\boldsymbol{0}_k, h^2 \boldsymbol{I}_k)}[(f(\boldsymbol{Z}))^2].$$

The lower-bound then follows. ∎

Now, we proceed to the proof of Lemma 4.3.

*Proof of Lemma 4.3 (i).* This is essentially Proposition 3.11 with $\sigma = 1$. ∎

*Proof of Lemma 4.3 (ii).*

$$\mu_f = \left(\mathop{\mathbb{E}}_{\boldsymbol{Z} \sim \mathcal{N}(\boldsymbol{0}_k, \boldsymbol{I}_k)}[f(\boldsymbol{Z})^6]\right)^{1/6} \stackrel{\text{Fact SM2.5}}{\leq} 5^{r/2} \cdot \left(\mathop{\mathbb{E}}_{\boldsymbol{Z} \sim \mathcal{N}(\boldsymbol{0}_k, \boldsymbol{I}_k)}\left[f(\boldsymbol{Z})^2\right]\right)^{1/2} = 5^{r/2}.$$

∎

*Proof of Lemma 4.3 (iii).* Let $f_h(\boldsymbol{z}) = f(h\boldsymbol{z})$, then:

(SM1.9) $$\overline{\nabla}_h f(\boldsymbol{\theta}) = \mathop{\mathbb{E}}_{\boldsymbol{Z} \sim \mathcal{N}(\boldsymbol{0}_k, \boldsymbol{I}_k)}[\nabla f(h\boldsymbol{Z} + \boldsymbol{U}^\mathsf{T} \boldsymbol{\theta})] = h^{-1} \mathop{\mathbb{E}}_{\boldsymbol{Z} \sim \mathcal{N}(\boldsymbol{0}_k, \boldsymbol{I}_k)}[\nabla f_h(\boldsymbol{Z} + \boldsymbol{U}^\mathsf{T} \boldsymbol{\theta}/h)].$$

For simplicity, define $S_j := \boldsymbol{U}_j^\mathsf{T} \boldsymbol{\theta}/h$ and note that $S_j \stackrel{iid}{\sim} \mathcal{N}(0, \sigma_\theta^2/h^2)$ for $j \in [k]$. Furthermore, we use $w_{h,\boldsymbol{\alpha}} := \langle f_h, H_{\boldsymbol{\alpha}} \rangle_{\mathcal{N}(\boldsymbol{0}_k, \boldsymbol{I}_k)}$ to denote the Hermite coefficient of $f_h$. By Lemma SM1.7, we have $\sum_{\boldsymbol{\alpha} \in \mathcal{I}_{r,k}} w_{h,\boldsymbol{\alpha}}^2 \leq (r+1)^{2k+(r/2)} \cdot \max\{h^{2r}, 1\}$.

137   We then proceed to calculate $\|\overline{\nabla}_h f(\boldsymbol{\theta})\|^4$ by (SM1.9) and (vii) from Fact SM2.1:

138
$$\left\| \mathop{\mathbb{E}}_{\boldsymbol{Z} \sim \mathcal{N}(\boldsymbol{0}_k, \boldsymbol{I}_k)} \nabla f_h(\boldsymbol{Z} + \boldsymbol{U}^\intercal \boldsymbol{\theta}/h) \right\|^2$$

139
$$= \sum_{i=1}^k \left( \sum_{\boldsymbol{\alpha} \in \mathcal{I}_{r,k}: \alpha_i \geq 1} w_{h,\boldsymbol{\alpha}} \prod_{j \neq i} \frac{1}{\sqrt{\alpha_j!}} S_j^{\alpha_j} \frac{\sqrt{\alpha_i}}{\sqrt{(\alpha_i - 1)!}} S_i^{\alpha_i - 1} \right)^2$$

140
$$\leq \sum_{i=1}^k \left( \sum_{\boldsymbol{\alpha} \in \mathcal{I}_{r,k}: \alpha_i \geq 1} w_{h,\boldsymbol{\alpha}}^2 \alpha_i \right) \left( \sum_{\boldsymbol{\alpha} \in \mathcal{I}_{r,k}: \alpha_i \geq 1} \prod_{j \neq i} \frac{S_j^{2\alpha_j}}{\alpha_j!} \cdot \frac{S_i^{2\alpha_i - 2}}{(\alpha_i - 1)!} \right)$$

141
$$\leq r \cdot \left( \sum_{\boldsymbol{\alpha} \in \mathcal{I}_{r,k}: \alpha_i \geq 1} w_{h,\boldsymbol{\alpha}}^2 \right) \left( \sum_{\boldsymbol{\alpha} \in \mathcal{I}_{r-1,k}} \prod_{j \in [k]} \frac{S_j^{2\alpha_j}}{\alpha_j!} \right)$$

142
$$\leq r \cdot (r+1)^{2k + (r/2)} \cdot \max\{h^{2r}, 1\} \cdot \left( \sum_{\boldsymbol{\alpha} \in \mathcal{I}_{r-1,k}} \prod_{j \in [k]} \frac{S_j^{2\alpha_j}}{\alpha_j!} \right).$$

143   By taking square on both sides,

144
$$\left\| \mathop{\mathbb{E}}_{\boldsymbol{Z} \sim \mathcal{N}(\boldsymbol{0}_k, \boldsymbol{I}_k)} \nabla f_h(\boldsymbol{Z} + \boldsymbol{U}^\intercal \boldsymbol{\theta}/h) \right\|^4$$

145
$$\leq (r+1)^{4k + r + 2} \cdot \max\{h^{4r}, 1\} \cdot \left( \sum_{\boldsymbol{\alpha} \in \mathcal{I}_{r-1,k}} \prod_{j \in [k]} \frac{S_j^{2\alpha_j}}{\alpha_j!} \right)^2$$

146
$$\leq (r+1)^{4k + r + 2} \cdot \max\{h^{4r}, 1\} \cdot |\mathcal{I}_{r-1,k}| \cdot \left( \sum_{\boldsymbol{\alpha} \in \mathcal{I}_{r-1,k}} \prod_{j \in [k]} \frac{S_j^{4\alpha_j}}{(\alpha_j!)^2} \right)$$

147 (SM1.10)
$$\leq (r+1)^{5k + r + 2} \cdot \max\{h^{4r}, 1\} \cdot \left( \sum_{\boldsymbol{\alpha} \in \mathcal{I}_{r-1,k}} \prod_{j \in [k]} \frac{S_j^{4\alpha_j}}{(\alpha_j!)^2} \right)$$

It remains upper-bound the expectation. By Fact SM2.3, we have:

$$\left( \frac{h}{\sigma_\theta} \right)^{4\alpha_j} \mathbb{E}\left[ \frac{S_j^{4\alpha_j}}{(\alpha_j!)^2} \right] = \frac{(4\alpha_j)!}{2^{2\alpha_j}(2\alpha_j)!(\alpha_j!)^2} \overset{\text{Stirling Approximation}}{\leq} 2^{4\alpha_j},$$

and hence by our choice $0 < \sigma_\theta < h/\sqrt{20}$:

$$\mathbb{E}\left[ \frac{S_j^{4\alpha_j}}{(\alpha_j!)^2} \right] \leq (2/\sqrt{20})^{4\alpha_j} \leq 1.$$

148　Finally, we have by (SM1.9):

149
$$\mathop{\mathbb{E}}_{\boldsymbol{\theta}\sim\mathcal{N}(\mathbf{0}_d,\sigma_\theta^2 \boldsymbol{I}_d)}\left[\|\overline{\nabla}_h f(\boldsymbol{\theta})\|_2^4\right] = h^{-4} \mathop{\mathbb{E}}_{\boldsymbol{\theta}\sim\mathcal{N}(\mathbf{0}_d,\sigma_\theta^2 \boldsymbol{I}_d)}\left[\left\|\mathop{\mathbb{E}}_{\boldsymbol{Z}\sim\mathcal{N}(\mathbf{0}_k,\boldsymbol{I}_k)}\nabla f_h(\boldsymbol{Z}+\boldsymbol{U}^\mathsf{T}\boldsymbol{\theta}/h)\right\|^4\right]$$

150
$$\leq (r+1)^{5k+r+2}\cdot\max\{h^{4r-4},h^{-4}\}\cdot\mathbb{E}\left[\sum_{\boldsymbol{\alpha}\in\mathcal{I}_{r-1,k}}\prod_{j\in[k]}\frac{S_j^{4\alpha_j}}{(\alpha_j!)^2}\right]$$

151
$$\leq (r+1)^{5k+r+2}\cdot\max\{h^{4r-4},h^{-4}\}\cdot|\mathcal{I}_{r-1,k}|$$

152
$$\leq (r+1)^{6k+r+2}\cdot\max\{h^{4r-4},h^{-4}\}.\qquad\blacksquare$$

153　　*Proof of Lemma 4.3 (iv).* Let $\boldsymbol{S}=(S_1,\ldots,S_k)$ with $S_i=\boldsymbol{U}_i^\mathsf{T}\boldsymbol{\theta}/h$. Clearly, we have $\boldsymbol{S}\sim$
154　$\mathcal{N}(\mathbf{0}_k,(\sigma_\theta/h)^2\boldsymbol{I}_k)$. By (SM1.9), recalling that $f_h(\boldsymbol{z}):=f(h\boldsymbol{z})$, we have

155
$$\overline{\nabla}_h f(\boldsymbol{\theta}) = \mathop{\mathbb{E}}_{\boldsymbol{Z}\sim\mathcal{N}(\mathbf{0}_d,h^2\boldsymbol{I}_d)}[\nabla f(\boldsymbol{U}^\mathsf{T}(\boldsymbol{Z}+\boldsymbol{\theta})))]$$

156
$$= h^{-1}\mathop{\mathbb{E}}_{\widetilde{\boldsymbol{Z}}\sim\mathcal{N}(\mathbf{0}_k,\boldsymbol{I}_k)}\left[\nabla f_h(\widetilde{\boldsymbol{Z}}+\boldsymbol{S}))\right]$$

157
$$= h^{-1}\left(\sum_{\boldsymbol{\alpha}\in\mathcal{I}_{r,k}:\alpha_i\geq 1}w_{h,\boldsymbol{\alpha}}\prod_{j\neq i}\frac{1}{\sqrt{\alpha_j!}}S_j^{\alpha_j}\frac{\sqrt{\alpha_i}}{\sqrt{(\alpha_i-1)!}}S_i^{\alpha_i-1}\right)_{i\in[k]}$$

158
$$=: h^{-1}\boldsymbol{g}(\boldsymbol{S}).$$

159　We aim to find a lower-bound for

160
$$\lambda_k(\overline{\boldsymbol{M}}) = \min_{\boldsymbol{\eta}\in\mathbb{R}^k:\|\boldsymbol{\eta}\|=1}\boldsymbol{\eta}^\mathsf{T}\mathop{\mathbb{E}}_{\boldsymbol{\theta}\sim\mathcal{N}(\mathbf{0}_k,\sigma_\theta^2\boldsymbol{I}_k)}\left[\overline{\nabla}_h f(\boldsymbol{\theta})\overline{\nabla}_h f(\boldsymbol{\theta})^\mathsf{T}\right]\boldsymbol{\eta}$$

161
$$= h^{-2}\min_{\boldsymbol{\eta}\in\mathbb{R}^k:\|\boldsymbol{\eta}\|=1}\mathop{\mathbb{E}}_{\boldsymbol{S}\sim\mathcal{N}(\mathbf{0}_k,(\sigma_\theta/h)^2\boldsymbol{I}_k)}\left[(\boldsymbol{\eta}^\mathsf{T}\boldsymbol{g}(\boldsymbol{S}))^2\right]$$

162　Here, we note that $\boldsymbol{\eta}^\mathsf{T}\boldsymbol{g}$ is a $k$-variate polynomial with a degree at most $p-1$:

163
$$\boldsymbol{\eta}^\mathsf{T}\boldsymbol{g}(\boldsymbol{s}) = \sum_{i\in[k]}\eta_i\left(\sum_{\boldsymbol{\alpha}\in\mathcal{I}_{r,k}:\alpha_i\geq 1}w_{h,\boldsymbol{\alpha}}\prod_{j\neq i}\frac{1}{\sqrt{\alpha_j!}}s_j^{\alpha_j}\frac{\sqrt{\alpha_i}}{\sqrt{(\alpha_i-1)!}}s_i^{\alpha_i-1}\right)$$

164　(SM1.11)
$$= \sum_{\boldsymbol{\alpha}\in\mathcal{I}_{r-1,k}}\left(\sum_{i\in[k]}\eta_i\cdot w_{h,\boldsymbol{\alpha}(i)}\cdot\sqrt{\alpha_i+1}\right)\cdot\prod_{j\in[k]}\frac{s_j^{\alpha_j}}{\sqrt{\alpha_j!}}.$$

165　where $\boldsymbol{\alpha}(i):=(\alpha_1,\ldots,\alpha_{i-1},\alpha_i+1,\alpha_{i+1},\ldots,\alpha_k)$ for $i\in[k]$. Thus, we can apply Lemma SM1.7
166　on $\boldsymbol{\eta}^\mathsf{T}\boldsymbol{g}$:

(SM1.12)
167
$$\mathop{\mathbb{E}}_{\boldsymbol{S}\sim\mathcal{N}(\mathbf{0}_k,(\sigma_\theta/h)^2\boldsymbol{I}_k)}\left[(\boldsymbol{\eta}^\mathsf{T}\boldsymbol{g}(\boldsymbol{S}))^2\right] \geq r^{-2k-(r-1)/2}\cdot(\sigma_\theta/h)^{2r-2}\cdot\mathop{\mathbb{E}}_{\boldsymbol{S}\sim\mathcal{N}(\mathbf{0}_k,\boldsymbol{I}_k)}\left[(\boldsymbol{\eta}^\mathsf{T}\boldsymbol{g}(\boldsymbol{S}))^2\right].$$

Next, we will establish a lower-bound for $\mathbb{E}_{\boldsymbol{S}\sim\mathcal{N}(\boldsymbol{0}_k, \boldsymbol{I}_k)}\left[(\boldsymbol{\eta}^\intercal \boldsymbol{g}(\boldsymbol{S}))^2\right]$. Let $\boldsymbol{D}_1 \in \mathbb{R}^{|\mathcal{I}_{r-1,k}|}$ and $\boldsymbol{D}_2 \in \mathbb{R}^{|\mathcal{I}_{r-1,k}|}$ be the coefficient vectors of $\boldsymbol{\eta}^\intercal \boldsymbol{g}(\boldsymbol{s})$ with respect to the Hermite basis and monomial basis, and $\boldsymbol{B}_{r-1,k} \in \mathbb{R}^{|\mathcal{I}_{r-1,k}| \times |\mathcal{I}_{r-1,k}|}$ be the invertible linear map that converts the monomial basis to the Hermite basis of the corresponding polynomial spaces. Then,

$$\mathbb{E}_{\boldsymbol{S}\sim\mathcal{N}(\boldsymbol{0}_k, \boldsymbol{I}_k)}\left[(\boldsymbol{\eta}^\intercal \boldsymbol{g}(\boldsymbol{S}))^2\right] = \|\boldsymbol{D}_1\|^2 = \|\boldsymbol{B}_{r-1,k}\boldsymbol{D}_2\|^2$$

(SM1.13)
$$\geq \lambda_{\min}\left(\boldsymbol{B}_{r-1,k}^\intercal \boldsymbol{B}_{r-1,k}\right) \cdot \|\boldsymbol{D}_2\|^2 \stackrel{\text{Fact SM2.2}}{\geq} 2^{-3(r-1)-k} \cdot \|\boldsymbol{D}_2\|^2.$$

Next, we establish a lower-bound of $\|\boldsymbol{D}_2\|^2$.

$$\|\boldsymbol{D}_2\|^2 = \sum_{\boldsymbol{\alpha}\in\mathcal{I}_{r-1,k}} \left(\prod_{j\in[k]} \alpha_j!\right)^{-1} \cdot \left(\sum_{i=1}^{k} \eta_i \cdot w_{h,\boldsymbol{\alpha}(i)} \cdot \sqrt{\alpha_i + 1}\right)^2$$

(SM1.14)
$$= (r!)^{-1} \cdot \sum_{\boldsymbol{\alpha}\in\mathcal{I}_{r-1,k}} \left(\sum_{i=1}^{k} \eta_i \cdot w_{h,\boldsymbol{\alpha}(i)} \cdot \sqrt{\alpha_i + 1}\right)^2.$$

For any unit $\boldsymbol{\eta}^\intercal \in \mathbb{R}^k$, we note that $\boldsymbol{\eta}^\intercal \nabla f$ is a polynomial of $k$ variables with degree at most $r-1$, we have by Lemma SM1.7 that:

$$\mathbb{E}_{\boldsymbol{Z}\sim\mathcal{N}(\boldsymbol{0}_k, h^2\boldsymbol{I}_k)}\left[(\boldsymbol{\eta}^\intercal \nabla f(\boldsymbol{Z}))^2\right] \geq r^{-2k-(r-1)/2} \cdot \min\{h^{2r-2}, 1\} \cdot \mathbb{E}_{\boldsymbol{Z}\sim\mathcal{N}(\boldsymbol{0}_k, \boldsymbol{I}_k)}\left[(\boldsymbol{\eta}^\intercal \nabla f(\boldsymbol{Z}))^2\right]$$

(SM1.15)
$$\geq r^{-2k-(r-1)/2} \cdot \min\{h^{2r-2}, 1\} \cdot \tau.$$

On the other hand, we have:

$$\mathbb{E}_{\boldsymbol{Z}\sim\mathcal{N}(\boldsymbol{0}_k, h^2\boldsymbol{I}_k)}\left[(\boldsymbol{\eta}^\intercal \nabla f(\boldsymbol{Z}))^2\right] = h^{-2} \mathbb{E}_{\boldsymbol{Z}\sim\mathcal{N}(\boldsymbol{0}_k, \boldsymbol{I}_k)}\left[(\boldsymbol{\eta}^\intercal \nabla f_h(\boldsymbol{Z}))^2\right]$$

(SM1.16)
$$= h^{-2} \sum_{\boldsymbol{\alpha}\in\mathcal{I}_{r-1,k}} \left(\sum_{i=1}^{k} \eta_i \cdot w_{h,\boldsymbol{\alpha}(i)} \cdot \sqrt{\alpha_i + 1}\right)^2,$$

where the last step we use multiple properties from Fact SM2.1. From (SM1.11)-(SM1.16), we conclude that:

(SM1.17) $\quad \lambda_k(\overline{\boldsymbol{M}}) \geq r^{-r-4k} \cdot ((r-1)!)^{-1} \cdot 2^{-3r-k+3} (\sigma_\theta/h)^{2r-2} \cdot \min\{h^{2r-2}, 1\} \cdot \tau.$  ∎

**SM2. Auxiliary Results Used in Section SM1.**

Fact SM2.1 (Facts about Hermite Polynomials, [SM1, SM2]). *Let $\{H_\ell\}_{\ell\in\mathbb{N}}$ be the normalized Probabilist's Hermite polynomials over $\mathbb{R}$, given by:*

$$H_\ell(z) = \sum_{i=1}^{\lfloor \ell/2 \rfloor} \frac{(-1)^i \sqrt{\ell!}}{2^i i! (l-2i)!} \cdot z^{\ell-2i}.$$

*Then,*

(i) $\{H_\ell\}_{\ell \in \mathbb{N}}$ is an orthonormal basis for $L_2(\mathcal{N}(0,1))$, i.e.,

$$\langle H_\ell, H_{\ell'} \rangle_{\mathcal{N}(0,1)} = \begin{cases} 1 & \text{if } \ell = \ell' \\ 0 & \text{otherwise} \end{cases}.$$

(ii) $\forall \ell \in \mathbb{N}$ and $\forall s \in \mathbb{R}$,

$$\mathbb{E}_{Z \sim \mathcal{N}}[H_\ell(Z+s)] = s^\ell/\sqrt{\ell!}.$$

(iii) $\forall \ell \in \mathbb{N}$ and $\forall h \in \mathbb{R}$,

$$H_\ell(hz) = \sqrt{\ell!} \sum_{i=0}^{\lfloor \ell/2 \rfloor} \frac{h^{\ell-2i}(h^2-1)^i}{2^i i! \sqrt{(\ell-2i)!}} H_{\ell-2i}(z).$$

(iv) $\forall \ell \geq 1$,

$$\frac{d}{dz} H_\ell(z) = \sqrt{\ell} H_{\ell-1}(z).$$

Furthermore, for $\boldsymbol{\alpha} \in \mathbb{N}^k$, let

$$H_{\boldsymbol{\alpha}}(\boldsymbol{z}) := \prod_{j=1}^k H_{\alpha_j}(z_j).$$

Then, the properties of univariate Hermite Polynomials can be generalized as follows:

(v) $\{H_{\boldsymbol{\alpha}}\}_{\boldsymbol{\alpha} \in \mathbb{N}^k}$ is an orthonormal basis for $L_2(\mathcal{N}(\mathbf{0}_k, \boldsymbol{I}_k))$, i.e.,

$$\langle H_{\boldsymbol{\alpha}}, H_{\boldsymbol{\alpha}'} \rangle_{\mathcal{N}(\mathbf{0}_k, \boldsymbol{I}_k)} = \begin{cases} 1 & \text{if } \boldsymbol{\alpha} = \boldsymbol{\alpha}' \\ 0 & \text{otherwise} \end{cases}.$$

(vi) $\forall f \in L_2(\mathcal{N}(\mathbf{0}_k, \boldsymbol{I}_k))$ can be written in terms of Hermite polynomials as

$$f(\boldsymbol{z}) = \sum_{\boldsymbol{\alpha}: \boldsymbol{\alpha} \in \mathbb{N}^k} w_{\boldsymbol{\alpha}} H_{\boldsymbol{\alpha}}(\boldsymbol{z}),$$

where $w_{\boldsymbol{\alpha}} = \langle f, H_{\boldsymbol{\alpha}} \rangle_{\mathcal{N}(\mathbf{0}_k, \boldsymbol{I}_k)}$. Furthermore, we have

$$\|f\|^2_{\mathcal{N}(\mathbf{0}_k, \boldsymbol{I}_k)} = \sum_{\boldsymbol{\alpha} \in \mathbb{N}^k} w_{\boldsymbol{\alpha}}^2.$$

(vii) $\forall \boldsymbol{s} \in \mathbb{R}^k$,

$$\mathbb{E}_{\boldsymbol{Z} \sim \mathcal{N}(\mathbf{0}_k, \boldsymbol{I}_k)}[\partial_i f(\boldsymbol{Z}+\boldsymbol{s})] = \sum_{\boldsymbol{\alpha} \in \mathbb{N}^k: \alpha_i \geq 1} w_{\boldsymbol{\alpha}} \prod_{j \neq i} \frac{1}{\sqrt{\alpha_j!}} s_j^{\alpha_j} \frac{\sqrt{\alpha_i}}{\sqrt{(\alpha_i-1)!}} S_i^{\alpha_i-1}.$$

(viii)

**Fact SM2.2 (Lemma 36 from [SM1]).** Let $\boldsymbol{B}_{r,k} \in \mathbb{R}^{|\mathcal{I}_{r,k}| \times |\mathcal{I}_{r,k}|}$ denote the invertible linear map that converts the monomial basis to the Hermite basis for the space of the polynomial of $k$ variables with degree at most $p$, then we have

$$\lambda_{\min}(\boldsymbol{B}_{r,k}^\top \boldsymbol{B}_{r,k}) \geq 2^{-3r-k}.$$

**Fact SM2.3** (Gaussian Moments). *For $X \sim \mathcal{N}(0,1)$, the $m^{th}$ moment is*

$$\mathbb{E}\left[X^m\right] = \begin{cases} 0 & m \text{ odd} \\ 2^{-(m/2)} \frac{m!}{(m/2)!} & m \text{ even} \end{cases}.$$

**Fact SM2.4.** *Suppose that $\epsilon$ is subgaussian with parameter $\sigma^2$, then $\mathbb{E}[|\epsilon|] \leq \sqrt{2\pi}\sigma$, and for any positive integer $q \geq 2$, the $q^{th}$ absolute moment satisfies the bound*

$$(\mathbb{E}\left[|\epsilon|^q\right])^{1/q} \leq \sigma e^{1/e} q^{1/2}.$$

**Fact SM2.5** (Gaussian Hypercontractivity, Theorem 11.23 from [SM2]). *Suppose $f$ is a polynomial with degree $p$ in $k$ variables. Given any integer $q \geq 2$, we have*

$$\left(\mathbb{E}_{\boldsymbol{Z} \sim \mathcal{N}(\boldsymbol{0}_k, \boldsymbol{I}_k)}[f(\boldsymbol{Z})^q]\right)^{1/q} \leq (q-1)^{r/2} \left(\mathbb{E}_{\boldsymbol{Z} \sim \mathcal{N}(\boldsymbol{0}_k, \boldsymbol{I}_k)}\left[f(\boldsymbol{Z})^2\right]\right)^{1/2}.$$

**Lemma SM2.6** (Truncated matrix Bernstein, Lemma 31 from [SM4]). *Consider a truncation level $R > 0$. Let $\{\boldsymbol{Z}_i\}_{1 \leq i \leq m}$ be a sequence of symmetric independent real random matrices with dimension $d \times d$ and $\boldsymbol{Z}'_i = \boldsymbol{Z}_i \mathbb{1}(\|\boldsymbol{Z}_i\|_{\text{op}} \leq R)$, then*

$$\mathbb{P}\left\{\left\|\frac{1}{m}\sum_{i=1}^m \boldsymbol{Z}_i - \mathbb{E}\,\boldsymbol{Z}_i\right\|_{\text{op}} \geq t\right\} \leq \mathbb{P}\left\{\left\|\frac{1}{m}\sum_{i=1}^m \boldsymbol{Z}'_i - \mathbb{E}\,\boldsymbol{Z}'_i\right\|_{\text{op}} \geq t - \Delta\right\} + \sum_{i=1}^m \mathbb{P}\{\|\boldsymbol{Z}_i\|_{\text{op}} > R\},$$

*where $\Delta \geq \|\mathbb{E}[\boldsymbol{Z}_i - \boldsymbol{Z}'_i]\|_{\text{op}}$. For $t \geq \Delta$, we further have*

$$\mathbb{P}\left\{\left\|\frac{1}{m}\sum_{i=1}^m \boldsymbol{Z}'_i - \mathbb{E}\,\boldsymbol{Z}'_i\right\|_{\text{op}} \geq t - \Delta\right\} \leq 2d \exp\left(-\frac{m^2(t-\Delta)^2}{\sigma^2 + 2Rm(t-\Delta)/3}\right),$$

*where $\sigma^2 \geq \|\sum_{i=1}^m \mathbb{E}[(\boldsymbol{Z}'_i - \mathbb{E}\,\boldsymbol{Z}'_i)^2]\|_{\text{op}}$.*

**Lemma SM2.7** (Block Matrix Inversion, 14.11. from [SM3]). *Let $\boldsymbol{A}$ be an invertible matrix with the following partition:*

$$\boldsymbol{A} = \begin{bmatrix} \boldsymbol{A}_{11} & \boldsymbol{A}_{12} \\ \boldsymbol{A}_{21} & \boldsymbol{A}_{22} \end{bmatrix},$$

*where $\boldsymbol{A}_{11}, \boldsymbol{A}_{22}$ are invertible. Then,*

$$\boldsymbol{A}^{-1} = \begin{bmatrix} (\boldsymbol{A}_{11} - \boldsymbol{A}_{12}\boldsymbol{A}_{22}^{-1}\boldsymbol{A}_{21})^{-1} & \boldsymbol{0} \\ \boldsymbol{0} & (\boldsymbol{A}_{22} - \boldsymbol{A}_{21}\boldsymbol{A}_{11}^{-1}\boldsymbol{A}_{12})^{-1} \end{bmatrix} \cdot \begin{bmatrix} \boldsymbol{I} & -\boldsymbol{A}_{12}\boldsymbol{A}_{22}^{-1} \\ -\boldsymbol{A}_{21}\boldsymbol{A}_{11}^{-1} & \boldsymbol{I} \end{bmatrix}.$$

## REFERENCES

[1] R. DUDEJA AND D. HSU, *Learning single-index models in gaussian space*, in Conference On Learning Theory, PMLR, 2018, pp. 1887–1930.
[2] R. O'DONNELL, *Analysis of boolean functions*, Cambridge University Press, 2014.
[3] G. A. F. SEBER, *A Matrix Handbook for Statisticians*, Wiley-Interscience, USA, 1st ed., 2007.
[4] N. TRIPURANENI, C. JIN, AND M. JORDAN, *Provable meta-learning of linear representations*, in International Conference on Machine Learning, PMLR, 2021, pp. 10434–10443.



\end{document}